\documentclass[runningheads]{llncs}
\usepackage{graphicx}
\usepackage{amsmath}
\usepackage{algorithm}
\usepackage{algorithmic}
\usepackage{subcaption}
\usepackage{amsfonts}
\usepackage{multirow}
\usepackage{diagbox}
\newcolumntype{C}[1]{>{\centering\arraybackslash}p{#1}}

\newtheorem{assumption}{Assumption}

\begin{document}
%
\title{A Framework for Evaluating Gradient Leakage Attacks in Federated Learning}
%
%
\author{Wenqi Wei\and
Ling Liu \and Margaret Loper \and
\\ Ka-Ho Chow \and Mehmet Emre Gursoy \and Stacey Truex \and Yanzhao Wu}
\authorrunning{W. Wei et al.}
\titlerunning{Federated Learning Privacy Leakage Evaluation Framework}
%
\institute{Georgia Institute of Technology, Atlanta GA 30332, USA \\
\email{wenqiwei@gatech.edu, ling.liu@cc.gatech.edu, margaret.loper@gtri.gatech.edu
 \\ \{staceytruex, memregursoy, khchow,yanzhaowu\}@gatech.edu, }
}
\maketitle              
\vspace{-0.2cm}
\begin{abstract}
Federated learning (FL) is an emerging distributed machine learning framework for collaborative model training with a network of clients (edge devices). FL offers default client privacy by allowing clients to keep their sensitive data on local devices and to only share local training parameter updates with the federated server. However, recent studies have shown that even sharing local parameter updates from a client to the federated server may be susceptible to gradient leakage attacks and intrude the client privacy regarding its training data. In this paper, we present a principled framework for evaluating and comparing different forms of client privacy leakage attacks. We first provide formal and experimental analysis to show how adversaries can reconstruct the private local training data by simply analyzing the shared parameter update from local training (e.g., local gradient or weight update vector). We then analyze how different hyperparameter configurations in federated learning and different settings of the attack algorithm may impact on both attack effectiveness and attack cost.
Our framework also measures, evaluates, and analyzes the effectiveness of client privacy leakage attacks under different gradient compression ratios when using communication efficient FL protocols. Our experiments also include some preliminary mitigation strategies to highlight the importance of providing a systematic attack evaluation framework towards an in-depth understanding of the various forms of client privacy leakage threats in federated learning and developing theoretical foundations for attack mitigation.

\end{abstract}

\vspace{-0.5cm}
\section{Introduction}

Federated learning enables the training of a high-quality ML model in a decentralized manner over a network of devices with unreliable and intermittent network connections~\cite{vanhaesebrouck2016decentralized,mcmahan2017communication,yang2019federated,konevcny2016federated,bonawitz2019towards,zhao2018federated}. In contrast to the scenario of prediction on edge devices, in which an ML model is first trained in a highly controlled Cloud environment and then downloaded to mobile devices for performing predictions, federated learning brings model training to the devices while supporting continuous learning on device. A unique feature of federated learning is to decouple the ability of conducting machine learning from the need of storing all training data in a centralized location~\cite{kamp2018efficient}.

Although federated learning by design provides the default privacy of allowing thousands of clients (e.g., mobile devices) to keep their original data on their own devices, while jointly learn a model by sharing only local training parameters with the server. Several recent research efforts have shown that the default privacy in FL is insufficient for protecting the underlaying training data from privacy leakage attacks by gradient-based reconstruction~\cite{geiping2020inverting,zhu2019deep,zhao2020idlg}. By intercepting the local gradient update shared by a client to the FL server before performing federated averaging~\cite{mcmahan2016federated,ma2015adding,kamp2018efficient}, the adversary can reconstruct the local training data with high reconstruction accuracy, and hence intrudes the client privacy and deceives the FL system by sneaking into client confidential training data illegally and silently, making the FL system vulnerable to client privacy leakage attacks (see Section~\ref{threat-model} on Threat Model for detail).

In this paper, we present a principled framework for evaluating and comparing different forms of client privacy leakage attacks. Through attack characterization, we present an in-depth understanding of different attack mechanisms and attack surfaces that an adversary may leverage to reconstruct the private local training data by simply analyzing the shared parameter updates (e.g., local gradient or weight update vector). Through introducing multi-dimensional evaluation metrics and developing evaluation testbed, we provide a measurement study and quantitative and qualitative analysis on how different configurations of federated learning and different settings of the attack algorithm may impact the success rate and the cost of the client privacy leakage attacks. Inspired by the attack effect analysis, we present some mitigation strategies with preliminary results to highlight the importance of providing a systematic evaluation framework for comprehensive analysis of the client privacy leakage threats and effective mitigation methods in federated learning.

The rest of the paper is organized as follows: Section~\ref{overview} presents the overview of the problem statement, including defining federated learning to establish some terminology required for our attack analysis, the threat model that describes the baseline assumptions about clients, FL server and adversaries, as well as the general formulation of client privacy leakage attack. Section~\ref{evaluation} describes three main components of our evaluation framework: the impact of attack parameter configurations on attack effectiveness and cost, the impact of FL hyperparameter configurations on attack effectiveness and cost, and the quantitative measurement metrics for attack effect and cost evaluation. Section~\ref{experiments} reports the measurement study and experimental analysis results on four benchmark image datasets. We conclude the paper with related work and a summary.


\vspace{-0.4cm}

\section{Problem Formulation}
\label{overview}

\subsection{Federated learning}

\label{section:fedlearning}


In federated learning, the machine learning task is decoupled from the centralized server to a set of $N$ client nodes. Given the unstable client availability~\cite{mcmahan2017federated}, for each round of federated learning, only a small subset of $K_t<$ clients out of all $N$ participants will be chosen to participate in the joint learning.

\textbf{Local Training at a Client}: Upon notification of being selected at round $t$, a client will download the global state $w(t)$ from the server, perform a local training computation on its local dataset and the global state, i.e., $w_k(t+1)=w_k(t)-\eta \nabla {w_k(t)}$, where $w_k(t)$ is the local model parameter update at round $t$ and $\nabla {w}$ is the gradient of the trainable network parameters. Clients can decide its training batch size $B_t$ and the number of local iterations before sharing.

\textbf{Update Aggregation at FL Server}: Upon receiving the local updates from all $K_t$ clients, the server incorporates these updates and update its global state, and initiates the next round of federated learning. Given that local updates can be in the form of either gradient or model weight update, thus two update aggregation implementations are the most representative:

\textbf{Distributed SGD.} At each round, each of the $K_t$ clients trains the local model with the local data and uploads the local gradients to the FL server. The server iteratively aggregates the local gradients from all $K_t$ clients into the global model, and check if the convergence condition of FL task is met and if not, it starts the next iteration round~\cite{lin2017deep,liu2020accelerating,yang2019federated,yao2019federated}.
$$w(t + 1) = w(t) - \eta \sum\nolimits_{k = 1}^{{K_t}} \frac{n_t}{n}  \nabla {w_k}(t),$$
where $\eta$ is the global learning rate and $ \frac{n_t}{n}$ is the weight of client $k$. Here we adopt the same notation as in reference~\cite{mcmahan2017communication} so that $n_k$ is the number of data points at client $k$ and $n$ indicates the amount of total data from all participating clients at round $t$.

\textbf{Federated averaging.} At each round, each of the $K_t$ clients uploads the local training parameter update to the FL server and the server iteratively performs a weighted average of the received weight parameters to update the global model, and starts the next iteration round $t+1$ unless it reaches the convergence~\cite{bagdasaryan2018backdoor,mcmahan2017communication}.
$$w(t + 1) = \sum\nolimits_{k = 1}^{{K_t}} \frac{n_t}{n}  {w_k}(t+1).$$
Let $\Delta  {w_k}(t)$ denote the difference between the model parameter update before the iteration of training and the model parameter update after the training for client $k$. Below is a variant of this method~\cite{geyer2017differentially}:
$$w(t + 1) = w(t) + \sum\nolimits_{k = 1}^{{K_t}} {\frac{n_t}{n} \Delta } {w_k}(t).$$

\textbf{Efficiency of FL Communication Protocol.}
The baseline communication protocol is used in many early federated learning implementations: the client sends a full vector of local training parameter update back to the FL server in each round. For federated training of large models on complex data, this step is known to be the bottleneck of Federated Learning. The communication efficient FL protocols have been proposed~\cite{mcmahan2017communication,konevcny2016federated}, which improves communication-efficiency of parameter update sharing by employing high precision vector compression mechanisms, such as structured updates and sketched updates. The former directly learns an update from a pre-specified structure, such as a low-rank matrix and random masks. The latter compresses the learned full vector of model parameter update to ensure a high compression ratio with a low-value loss before sending it to the server.

Our framework will study the impact of different configurations of FL hyperparameters on the success rate and cost of privacy leakage attacks. For instance, we will show that both baseline protocol and communication efficient protocol (e.g., sketched updates) are vulnerable to client gradient leakage attacks (see Section~\ref{sec3.2} and Section~\ref{sec4.2}).


\subsection{Threat Model}

\label{threat-model}


In an FL system, clients are the most vulnerable attack surface for the client privacy leakage (CPL) attack because the client is the one sending its local training parameter update to the FL server.
We assume that clients can be compromised in a limited manner: an adversary cannot gain access to the private training data but may intercept the local parameter update to be sent to the FL server and be able to access and run the saved local model executable (checkpoint data) on the compromised client to launch white-box gradient leakage attack.

On the other hand, we assume that the federated server is honest but curious. Namely, the FL server will honestly perform the aggregation of local parameter updates and manage the iteration rounds for jointly learning. However, the FL server may be curious and may analyze the periodic updates from certain clients to perform client privacy leakage attacks and gain access to the private training data of the victim clients. Given that the gradient-based aggregation and model weight-based aggregation are mathematically equivalent, one can obtain the local model parameter difference from the local gradient and the local training learning rate. In this paper, gradient-based aggregation is used without loss of generality. It is worth noting that even if the network connection between client and server is secure, the client privacy leakage attack can happen on a compromised client before the local parameter update is prepared for upload to the server.

Finally, distributed training can happen in either federated scenario with a central server~\cite{li2014scaling} or a decentralized scenario where clients are connected via a peer to peer network~\cite{roy2019braintorrent}. In the decentralized setting, a client node will first perform the local computation to update its local model parameter and then sends the updated gradients to its neighbor nodes. The client privacy leakage attack can happen in both scenarios on the compromised client nodes. The federated scenario is assumed in this paper.

\subsection{The Client Privacy Leakage (CPL) Attack: An Overview}
\label{sec:2.3}

The client privacy leakage attack is a gradient-based feature reconstruction attack, in which the attacker can design a gradient-based reconstruction learning algorithm that will take the gradient update at round $t$, say $\nabla w_k(t)$, to be shared by the client, to reconstruct the private data used in the local training computation. For federated learning on images or video clips, the reconstruction algorithm will start by using a dummy image of the same resolution as its attack initialization seed, and run a test of this attack seed on the intermediate local model, to compute a gradient loss using a vector distance loss function between the gradient of this attack seed and the actual gradient from the client local training. The goal of this reconstruction attack is to iteratively add crafted small noises to the attack seed such that the generated gradient from this reconstructed attack seed will approximate the actual gradient computed on the local training data. The reconstruction attack terminates when the gradient of the attack seed reconstructed from the dummy initial data converges to the gradient of the training data. When the gradient-based reconstruction loss function is minimized, the reconstructed attack data will also converge to the training data with high reconstruction confidence.
Algorithm~\ref{algo1} gives a sketch of the client privacy leakage attack method.

\vspace{-0.2cm}

\begin{algorithm}
\caption{Gradient-based Reconstruction Attack}
\begin{algorithmic}[1]
\STATE  \textbf{Inputs:} \\
$f(x; w(t))$: Differentiable learning model,
$\nabla w_k(t)$: gradients produced by the local training on private training data $(x; y)$ at client $k$,
$w(t)$, $w(t+1)$: model parameters before and after the current local training on $(x; y)$,
$\eta_k$ learning rate of local training \\
\textbf{Attack configurations}:
$INIT(x.type)$: attack initialization method,
$\mathbb{T}$: attack termination condition,
$\eta'$: attack optimization method,
$\alpha$: regularizer ratio. \
\STATE \textbf{Output:} reconstructed training data  $(x_{rec}; y_{rec})$ \
\STATE \textbf{Attack procedure} \
\IF{$w_k(t+1)$:}
\STATE  $\Delta w_{k}(t) \leftarrow w_k(t+1) - w(t)$ \
\STATE  $\nabla w_k(t) \leftarrow \frac{\Delta w_{k}(t)}{\eta_k}$ \
\ENDIF
\STATE  $x^0_{rec} \leftarrow \textit{INIT}(x.type)$ \
\STATE $y_{rec} \leftarrow  \arg\min_i(\nabla_i w_k(t)$) \
\FOR{$\tau$ in $\mathbb{T}$}
\STATE  $\nabla w^{\tau}_{att}(t) \leftarrow  \frac{\partial loss(f(x^{\tau}_{rec},w(t)), y_{rec})}{\partial w(t)}$ \
\STATE $D^{\tau} \leftarrow  ||\nabla w^{\tau}_{att}(t)-\nabla w_k(t)||^2 + \alpha||f(x^{\tau}_{rec},w(t))-y_{rec}||^2$  \
\STATE $ x^{\tau+1}_{rec} \leftarrow x^{\tau}_{rec} - \eta'  \frac{\partial D^{\tau}}{\partial x^{\tau}_{rec}}$
\ENDFOR
\end{algorithmic}
\label{algo1}
\end{algorithm}

\vspace{-0.2cm}

In Algorithm~\ref{algo1}, line~4-6 convert the weight update to gradient when the weight update is shared between the FL server and the client. The learning rate $\eta_k$ for local training is assumed to be identical across all clients in our prototype system. Line~8 invokes the dummy attack seed initialization, which will be elaborated in Section~\ref{sec:3.1attackconfig}. Line~9 is to get the label from the actual gradient shared from the local training. Since the local training update towards the ground-truth label of the training input data should be the most aggressive compared to other labels, the sign of gradient for the ground-truth label of the private training data will be different than other classes and its absolute value is usually the largest.
Line~10-14 presents the iterative reconstruction process that produces the reconstructed private training data based on the client gradient update. If the reconstruction algorithm converges, then the client privacy leakage attack is successful, and else the CPL attack is failed. Line~12-13 show that when the $L_2$ distance between the gradients of the attack reconstructed data and the actual gradient from the private training data is minimized, the reconstructed attack data from the dummy seed converges to the private local training data, leading to the client privacy leakage. In line~12, a label-based regularizer is utilized to improve the stability of the attack optimization. An alternative way to reconstruct the label of the private training data is to initialize a dummy label and feed it into the iterative approximation algorithm for attack optimization~\cite{zhu2019deep}, in a similar way as the content reconstruction optimization. Figure~\ref{fig:sec2.3example} provides a visualization of four illustrative example attacks over four datasets: LFW~\cite{huang2008labeled}, CIFAR100~\cite{krizhevsky2009learning}, MNIST~\cite{lecun1998mnist}, and CIFAR10~\cite{krizhevsky2009learning}.
\vspace{-0.7cm}
\begin{figure}[ht]
\centerline{\includegraphics[scale=.45]{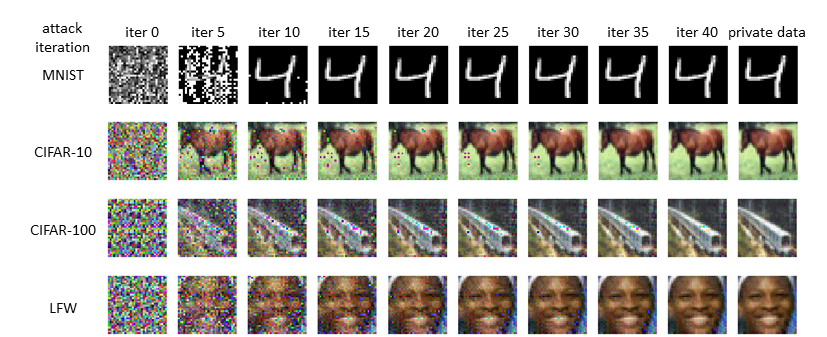}}
\caption{\small Example illustration of the Client Privacy Leakage attack}
\label{fig:sec2.3example}
\vspace{-0.4cm}
\end{figure}
\vspace{-0.2cm}


We make two interesting observations from Algorithm~\ref{algo1}. {\em First}, multiple factors in the attack method could impact the attack success rate (ASR) of the client privacy leakage attack, such as the dummy attack seed data initialization method (line~8), the attack iteration termination condition ($\mathbb{T}$), the selection of the gradient loss (distance) function (line~12), the attack optimization method (line~13). {\em Second}, the configuration of some hyperparameters in federated learning may also impact the effectiveness and cost of the CPL attack, including batch size, training data resolution, choice of activation function, and whether the gradient update is uploaded to the FL server using baseline communication protocol or a communication efficient method. In the next section, we present the design of our evaluation framework to further characterize the client privacy leakage attack of different forms and introduce cost-effect metrics to measure and analyze the adverse effect and cost of CPL attacks. By utilizing this framework, we provide a comprehensive study on how different attack parameter configurations and federated learning hyperparameter configurations may impact the effectiveness of the client privacy leakage attacks.

\vspace{-0.2cm}
\section{Evaluation Framework}
\label{evaluation}

\subsection{Attack Parameter Configuration}
\label{sec:3.1attackconfig}

\subsubsection{Attack Initialization:}

We first study how different methods for generating the dummy attack seed data may influence the effectiveness of a CPL attack in terms of reconstruction quality or confidence as well as reconstruction time and convergence rate. A straightforward dummy data initialization method is to use a random distribution in the shape of dummy data type and we call this baseline the random initialization (CPL-random). Although random initialization is also used in~\cite{zhu2019deep,zhao2020idlg,geiping2020inverting}, this is, to the best of our knowledge, the first study on variations of attack initiation methods. To understand the role of random seed in the CPL attack, it is important to understand the difference of the attack reconstruction learning from the normal deep neural network (DNN) training. In a DNN training, it takes as the training input both the fixed data-label pairs and the initialization of the learnable model parameters, and iteratively learn the model parameters until the training converges, which minimizes the loss with respect to the ground truth labels. In contrast, the CPL attack performs reconstruction attack by taking a dummy attack seed input data, a fixed set of model parameters, such as the actual gradient updates of a client local training, and the gradient derived label as the reconstructed label $y_{rec}$, its attack algorithm will iteratively reconstruct the local training data used to generate the gradient, $\nabla w_k(t)$, by updating the dummy synthesized seed data, following the attack iteration termination condition $\mathbb{T}$, denoted by  $\{x_{rec}^0,x_{rec}^1,...{x_s}^{\mathbb{T}}\} \in \mathbb{R}^d$, such that the loss between the gradient of the reconstructed data $x_{rec}^i$ and the actual gradient $\nabla w_k(t)$ is minimized. Here $x_{rec}^0$ denotes the initial dummy seed.

\begin{theorem} (\textbf{CPL Attack Convergence Theorem }) Let $x_{rec}^*$ be the optimal synthesized data for $f(x)$ and attack iteration $t\in \{0,1,2,...T\}$. Given the convexity and Lipschitz-smoothness assumption, the convergence of the gradient-based reconstruction attack is guaranteed with:
\begin{equation}
    f(x_{rec}^{\mathbb{T}})-f(x_{rec}^*) \leq  \frac{2L||x_{rec}^0-x_{rec}^*||^2}{\mathbb{T}}.
\end{equation}
\label{theorem1}
\end{theorem}

\vspace{-0.4cm}

The above CPL Attack Convergence theorem is derived from the well-established Convergence Theorem of Gradient Descent~\cite{convergencetheorem}. Due to the limitation of space, the formal proof of Theorem~\ref{theorem1} is provided in the appendix. Note that the convexity assumption is generally true since the $d$-dimension trainable synthesized data can be seen as a one-hidden-layer network with no activation function. The fixed model parameters are considered as the input with optimization of the least square estimation problem as stated in Line~12 of Algorithm~\ref{algo1}.

According to the CPL Attack Convergence Theorem, the convergence of the CPL attack is closely related to the initialization of the dummy data $x_{rec}^0$. This motivates us to investigate different ways to generate dummy attack seed data. Concretely, we argue that different random seeds may have different impacts on both reconstruction learning efficiency (confidence) and reconstruction learning convergence (time or the number of iteration steps). Furthermore, using geometrical initialization as those introduced in~\cite{rossi1994geometrical} not only can speed up the convergence but also ensure attack stability. Consider a single layer of the neural network: $g(x) = \sigma(wx+b)$, a geometrical initialization considers the form $g(x) = \sigma(w_*(x-b_*)$ instead of directly initialing $w$ and $b$ with random distribution. For example, according to~\cite{zhang1992wavelet}, the following partial derivative of the geometrical initialization.
\vspace{-0.2cm}
\begin{equation}
    \frac{{\partial g}}{{\partial w_*}} = \sigma '(w_*(x - b_*))(x - b_*),
    \vspace{-0.3cm}
\end{equation}
is more independent from translation of the input space than $\frac{{\partial g}}{{\partial w}} = \sigma '(wx + b)x,$ and is therefore more stable.

\vspace{-0.7cm}

\begin{figure}[ht]
\centerline{\includegraphics[scale=.32]{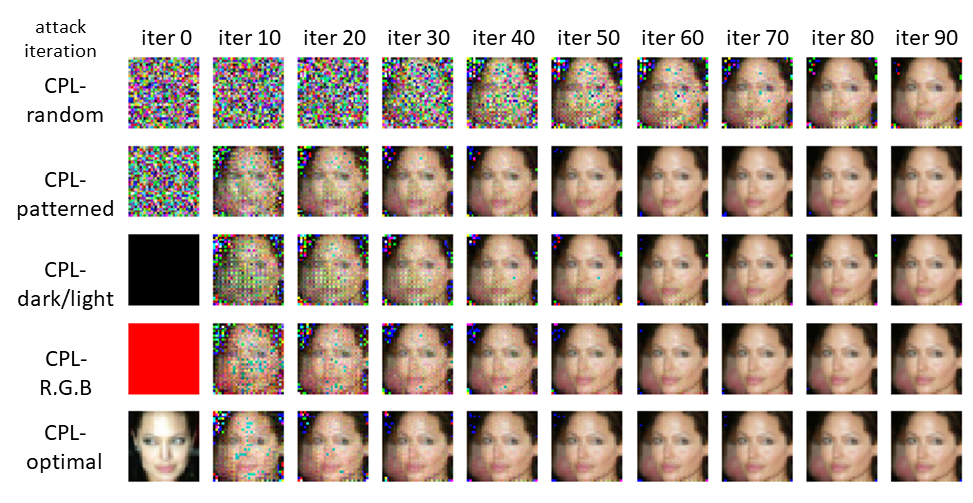}}
\caption{\small Visualization of different initialization}
\label{fig:init_vis}
\vspace{-0.4cm}
\end{figure}

\vspace{-0.2cm}

Figure~\ref{fig:init_vis} provides a visualization of five different initialization methods and their impact on the CPL attack in terms of reconstruction quality and convergence (\#iterations). In addition to CPL-random,  CPL-patterned is a method that uses patterned random initialization. We initialize a small portion of the dummy data with a random seed and duplicate it to the entire feature space. An example of the portion can be 1/4 of the feature space.
CPL-dark/light is to use a dark (or light) seed of the input type (size), whereas CPL-R.G.B. is to use red or green or blue solid color seed of the input type (size). CPL-optimal refers to the theoretical optimal initialization method, which uses an example from the same class as the private training data that the CPL attack aims to reconstruct. We observe from Figure~\ref{fig:init_vis} that CPL-patterned, CPL-R.G.B., and CPL-dark/light can outperform CPL-random with faster attack convergence and more effective reconstruction confidence.
We also include CPL-optimal to show that different CPL initializations can effectively approximate the theoretical optimal initialization in terms of reconstruction effectiveness.

Figure~\ref{fig:init_seed} shows that the CPL attacks are highly sensitive to the choice of random seeds. We conduct this set of experiments on LFW and CIFAR100 and both confirm consistently our observations: different random seeds lead to diverse convergence processes with different reconstruction quality and confidence. From Figure~\ref{fig:init_seed}, we observe that even with the same random seed, attack with patterned initialization is much more efficient and stable than the CPL-random. Moreover, there are situations where the private label of the client training data is successfully reconstructed but the private content reconstruction fails (see the last row for both LFW and CIFAR100).

\vspace{-0.6cm}

\begin{figure}[ht]
\begin{minipage}{0.49\linewidth}
 \centerline{\includegraphics[scale=.38]{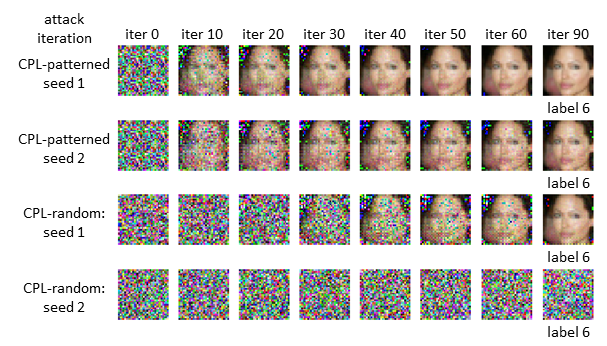}}
 \subcaption{\small LFW}
 \label{fig:init_seed_lfw}
\end{minipage}
\begin{minipage}{0.49\linewidth}
 \centerline{\includegraphics[scale=.38]{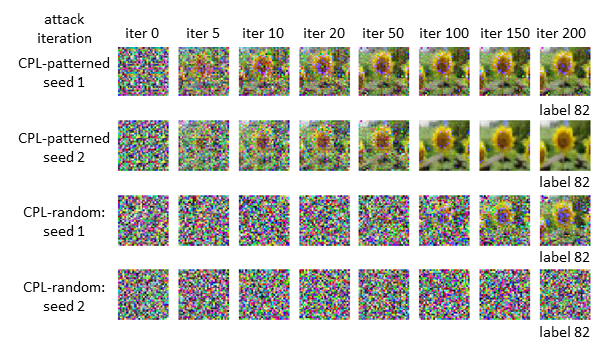}}
  \subcaption{\small CIFAR100}
  \label{fig:init_seed_cifar100}
\end{minipage}
\caption{\small Effect of different random seed}
\label{fig:init_seed}
\vspace{-0.4cm}
\end{figure}


\subsubsection{Attack Termination Condition:} The effectiveness of a CPL attack is also sensitive to the setting of the attack termination condition. Recall Section~\ref{sec:2.3} and Algorithm~\ref{algo1}, there are two decision factors for termination. One is the maximum attack iteration and the other is the $L_2$-distance threshold of the gradient loss, i.e., the difference between the gradient from the reconstructed data and the actual gradient from the local training using the private local data.

Consider the configuration of the maximum number of attack iterations, Table~\ref{table:termination} compares the six different settings of attack iteration termination condition. It shows that when it is too small, such as 10 or 20, no matter how to optimize the attack initialization method, the CPL attack will fail for all or most of the datasets. However, when it is set to sufficiently larger, say 100 or 300, choosing a good attack initialization method matters more significantly. For example, CPL-patterned can get an attack success rate higher than the CPL-random for all three datasets. Also, CPL-random shows unstable attack performance: For LFW, it has slightly higher ASR when choosing the termination of 300 over that of 100 iterations. However, for CIFAR10 and CIFAR100, it has significantly higher ASR when choosing the termination of 300 iterations over that of 100. This set of experiments also shows the good configuration of the total number of termination iterations can be a challenging problem.

The second factor used for setting the attack termination condition is the $L_2$-distance threshold of the gradient difference between the reconstructed data and the private local data. This factor is dataset-dependent, and our experiments with the four benchmark datasets show that a $L2$-distance threshold, such as 0.0001, is a good option in terms of generalization.

\vspace{0.3cm}

\noindent
\begin{minipage}{\textwidth}
\begin{minipage}{0.49\linewidth}
\scalebox{0.70}{
\begin{tabular}{|c|c|c|c|c|c|c|c|}
\hline
\multicolumn{2}{|c|}{maximum attack iteration} & 10 & 20   & 30   & 50    & 100   & 300   \\ \hline
\multirow{2}{*}{LFW}         & CPL-patterned   & 0  & 0.34 & 0.98 & 1     & 1     & 1     \\ \cline{2-8}
                             & CPL-random            & 0  & 0    & 0    & 0.562 & 0.823 & 0.857 \\ \hline
\multirow{2}{*}{CIFAR10}     & CPL-patterned   & 0  & 0.47 & 0.93 & 0.973 & 0.973 & 0.973 \\ \cline{2-8}
                             & CPL-random            & 0  & 0    & 0    & 0     & 0.356 & 0.754 \\ \hline
\multirow{2}{*}{CIFAR100}    & CPL-patterned   & 0  & 0    & 0.12 & 0.85  & 0.981 & 0.981 \\ \cline{2-8}
                             & CPL-random            & 0  & 0    & 0    & 0     & 0.23  & 0.85  \\ \hline
\end{tabular}
}
 \captionof{table}{\small Effect of termination condition}
 \label{table:termination}
\end{minipage}
\hspace{0.5cm}
\begin{minipage}{0.49\linewidth}
 \centerline{\includegraphics[scale=.35]{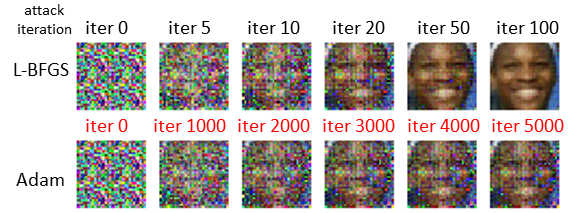}}
  \captionof{figure}{\small Effect of attack optimization}
  \label{fig:attack_optimization}
\end{minipage}
\vspace{-0.4cm}
\end{minipage}

\subsubsection{Gradient Loss (Distance) Function:} In the CPL attack Algorithm~1, we use $L_2$ distance function as the gradient loss function. There are other alternative vector distance functions, such cosine similarity, entropy, and so forth, can be applied. ~\cite{janocha2017loss} has studied the impact of different distance functions on training efficiency. It shows that using some loss functions might lead to slower training, while others can be more robust to noise in the training set labeling and also slightly more robust to noise in the input space.

\vspace{-0.5cm}

\subsubsection{Attack optimization:} Optimization methods, such as Stochastic Gradient descent~\cite{robbins1951stochastic}, Momentum~\cite{qian1999momentum}, Adam~\cite{kingma2014adam}, and Adagrad~\cite{duchi2011adaptive} can be used to iteratively update the dummy data during the reconstruction of a CPL attack. While the first-order optimization techniques are easy to compute and less time consuming, the second-order techniques are better in escaping the slow convergence paths around the saddle points~\cite{battiti1992first}. Figure~\ref{fig:attack_optimization} shows a comparison of L-BFGS~\cite{fletcher2013practical} and Adam and their effects on the CPL-patterned attack for LFW dataset. It shows that choosing an appropriate optimizer can significantly improve attack effectiveness. In the rest of the paper, L-BFGS is used in our experiments.

\vspace{-0.2cm}

\subsection{Hyperparameter Configurations in Federated learning}
\label{sec3.2}

\subsubsection{Batch size:} Given that all forms of CPL attack methods are reconstruction learning algorithms that iteratively learn to reconstruct the private training data by inferencing over the actual gradient to perform iterative updates on the dummy attack seed data, it is obvious that a CPL attack is most effective when working with the gradient generated from the local training data of batch size 1. Furthermore, when the input data examples in a batch of size $B$ belongs to only one or two classes, which is often the case for mobile devices and the non-i.i.d distribution of the training data~\cite{zhao2018federated}, the CPL attacks can effectively reconstruct the training data of the entire batch. This is especially true when the dataset has low inter-class variation, e.g., face and digit recognition. Figure~\ref{fig:batchsize} shows the visualization of performing a CPL-patterned attack on the LFW dataset with four different batch sizes.

\vspace{-0.6cm}

\begin{figure}[ht]
 \centerline{\includegraphics[scale=.40]{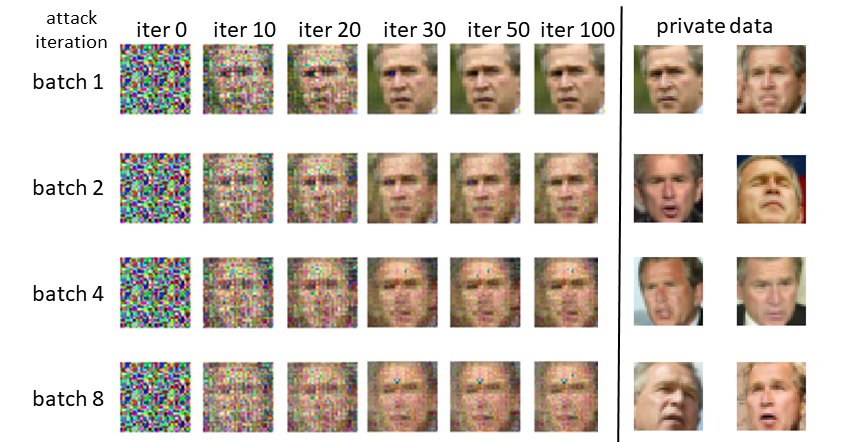}}
  \caption{\small Effect of batch size in CPL-patterned attacks on LFW}
  \label{fig:batchsize}
\vspace{-0.4cm}
\end{figure}

\vspace{-0.8cm}

\subsubsection{Training Data Resolution:} In contrast to the early work~\cite{zhu2019deep} that fails to attack images of resolution higher than $64\times 64$, we argue that the effectiveness of the CPL attack is mainly attributed to the model overfitting to the training data. In order to handle higher resolution training data, we double the number of filters in all convolutional layers to build a more overfitted model. Figure~\ref{fig:scaling} shows the scaling results of CPL attack on the LFW dataset with input data size of $32\times32$,$64\times 64$, and $128\times 128$. CPL-random requires a much larger number of attack iterations in order to succeed the attack with high reconstruction performance. CPL-patterned is a significantly more effective attack for all three different resolutions with $3$ to $4\times$ reduction in the attack iterations compared to CPL-random. We also provide an example of attacking the $512\times 512$ Indiana University Chest X-Rays image of very high resolution in Figure~\ref{fig:x_ray}.

\vspace{-0.6cm}

\noindent
\begin{figure}[ht]
\begin{minipage}{0.49\linewidth}
 \centerline{\includegraphics[scale=.28]{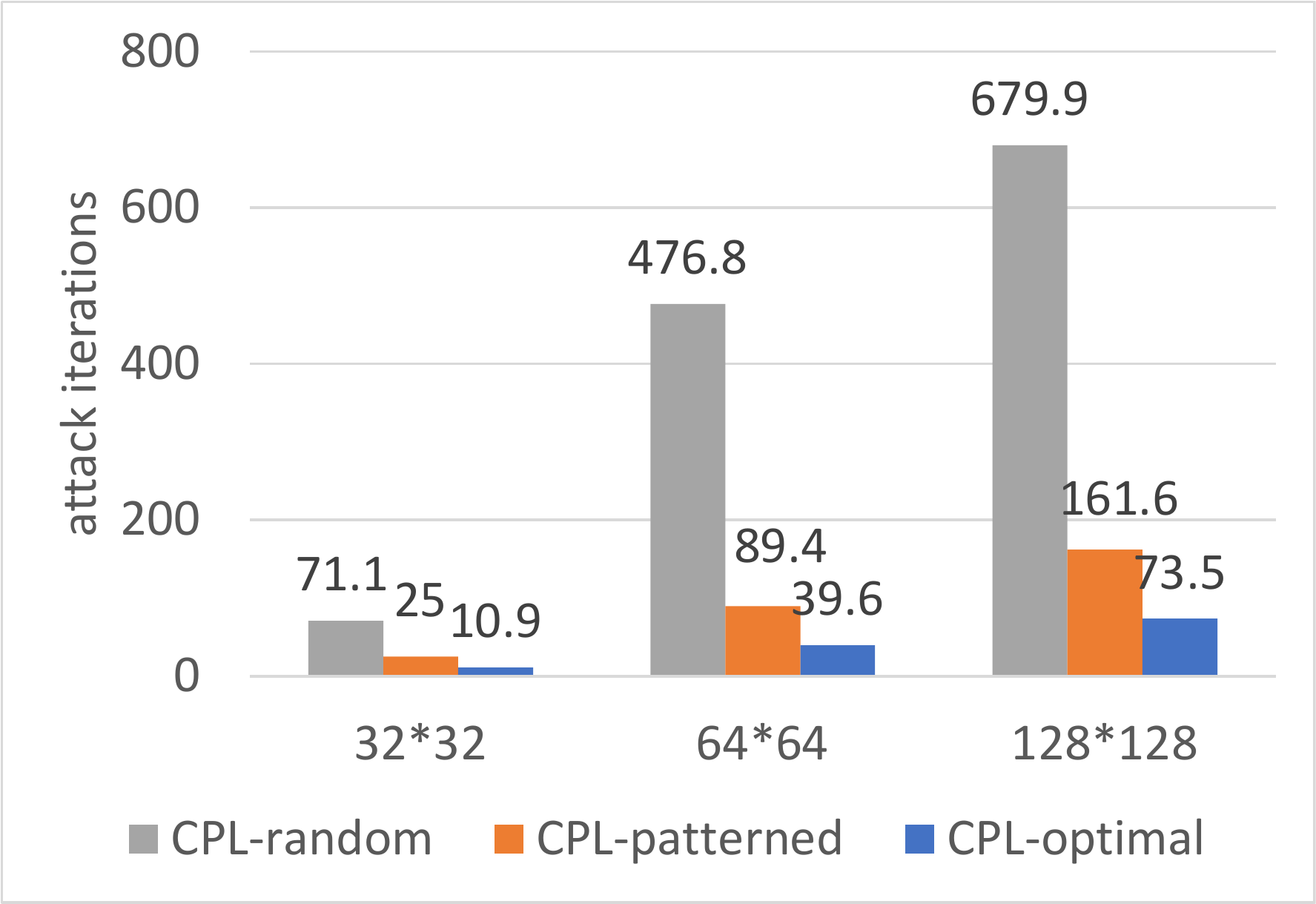}}
 \caption{{\small Effect of  data scaling (LFW)}}
\label{fig:scaling}
\vspace{-0.4cm}
\end{minipage}
\begin{minipage}{0.49\linewidth}
\vspace{0.3cm}
  \centerline{\includegraphics[scale=.61]{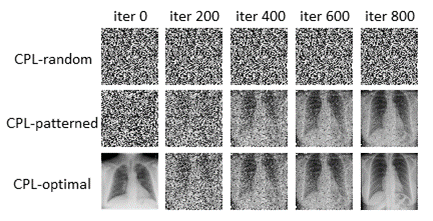}}
 \caption{{\small Attacking $512\times512$ X-ray image}}
\label{fig:x_ray}
\end{minipage}
\end{figure}

\vspace{-1.2cm}

\begin{figure}[ht]
\begin{minipage}{0.49\linewidth}
 \centerline{\includegraphics[scale=.30]{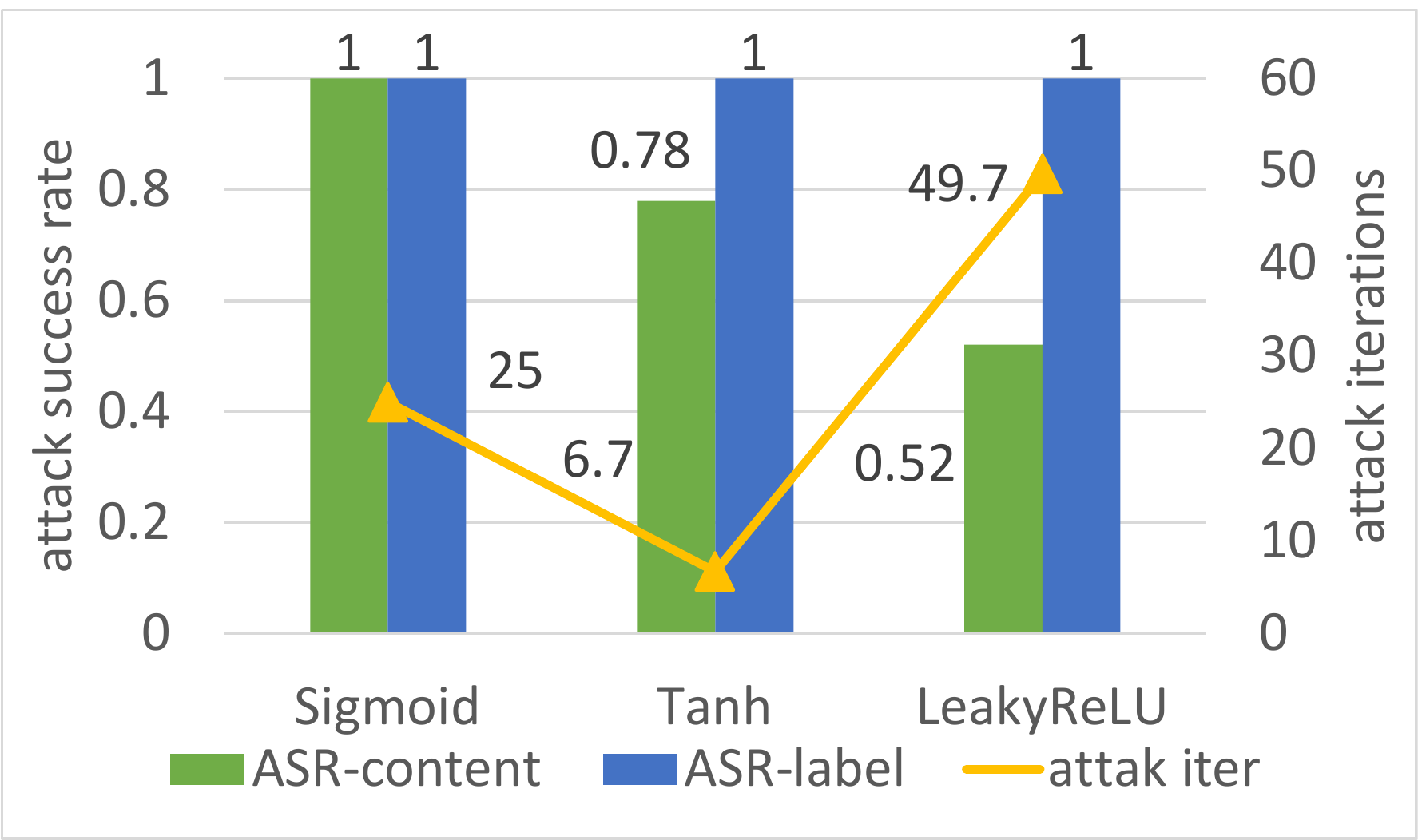}}
 \subcaption{\small LFW}
 \label{fig:activation_lfw}
\end{minipage}
\begin{minipage}{0.49\linewidth}
 \centerline{\includegraphics[scale=.30]{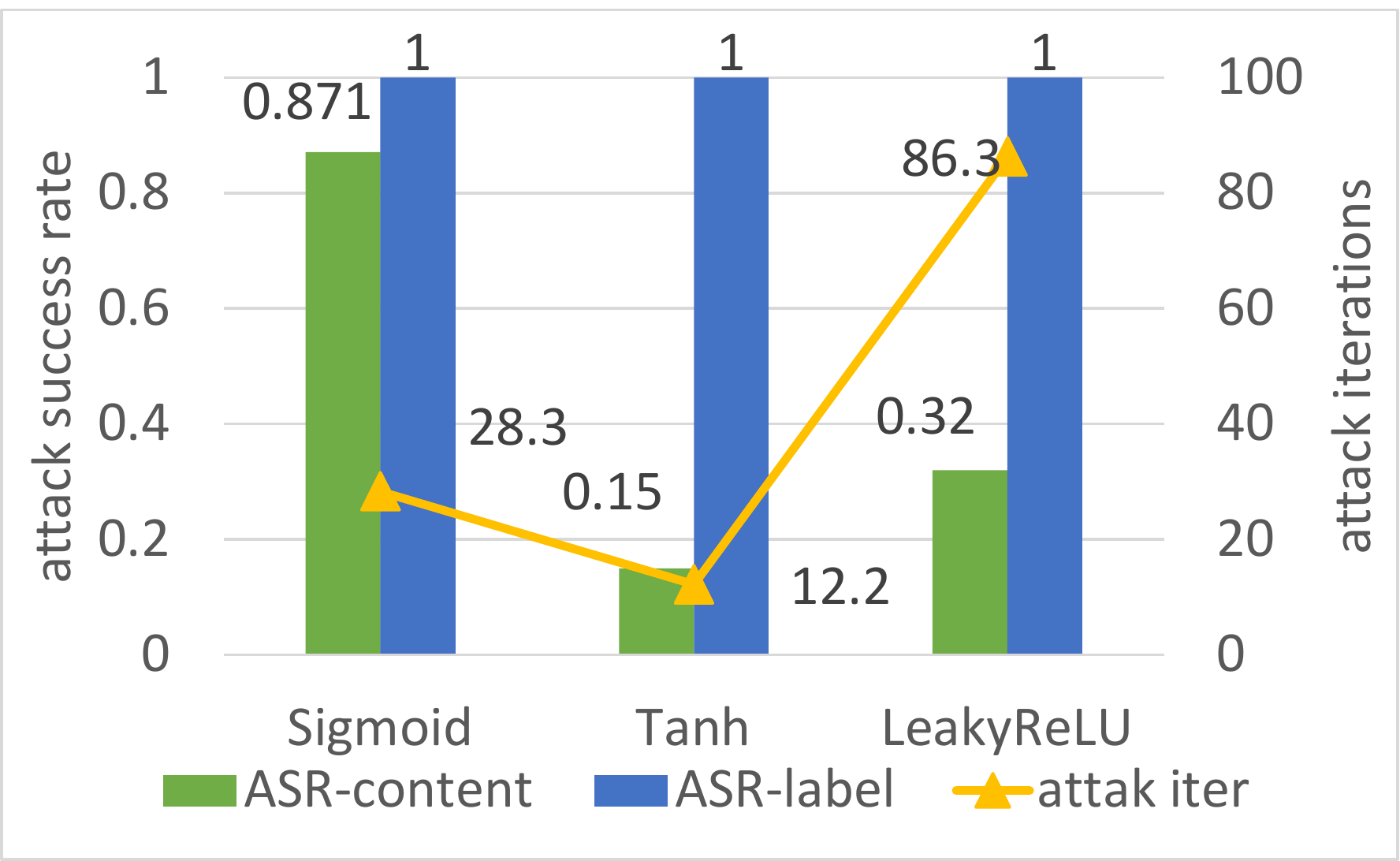}}
  \subcaption{\small CIFAR10}
  \label{fig:activation_cifar10}
\end{minipage}
\begin{minipage}{0.49\linewidth}
 \centerline{\includegraphics[scale=.30]{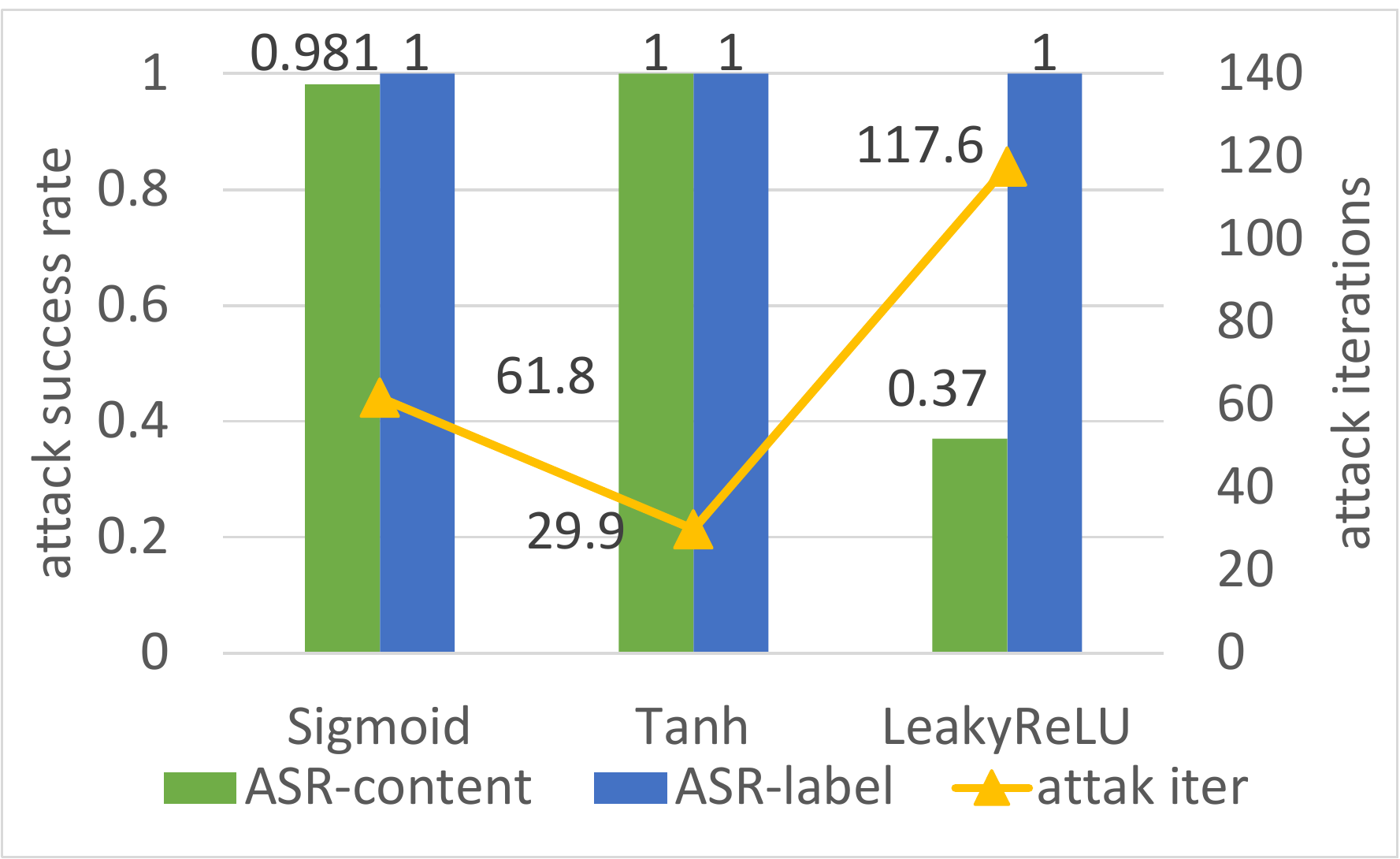}}
 \subcaption{\small CIFAR100}
 \label{fig:activation_cifar100}
\end{minipage}
\begin{minipage}{0.49\linewidth}
 \centerline{\includegraphics[scale=.30]{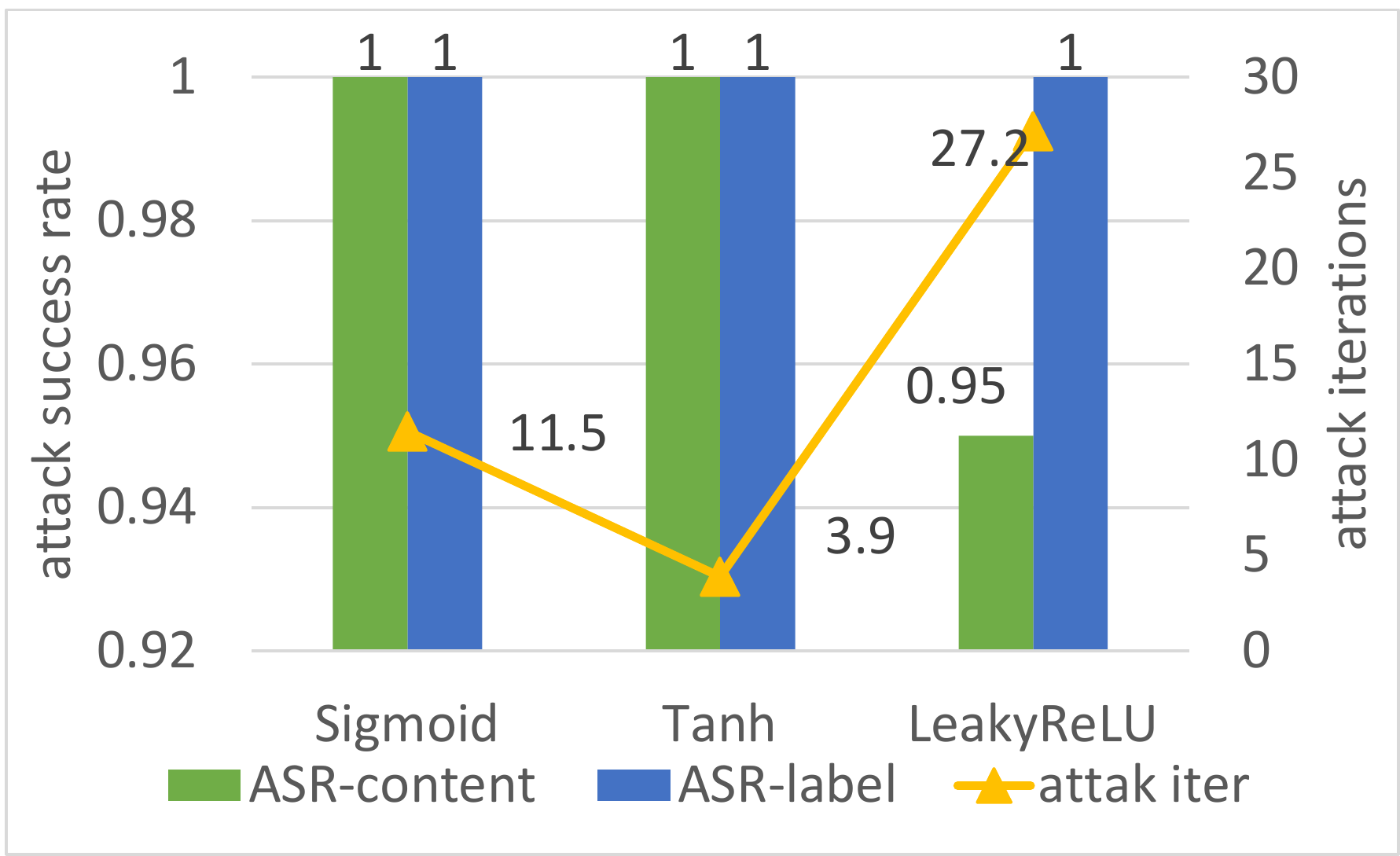}}
  \subcaption{\small MNIST}
  \label{fig:activation_mnist}
\end{minipage}
\caption{\small Effect of activation function on the CPL attack}
\label{fig:activation}
\vspace{-0.4cm}
\end{figure}

\vspace{-0.6cm}

\subsubsection{Activation Function:} The next hyperparameter of FL is the activation function used in model training. We show that the performance of the CPL attacks is highly related to the choice of the activation function. Figure~\ref{fig:activation} compares the attack iterations and attack success rate of CPL-patterned attack with three different activation functions: Sigmoid, Tanh, and LeakReLU.
We observe that ReLU naturally prevents the full reconstruction of the training data using gradient because the gradient of the negative part of ReLU will be 0, namely,  that part of the trainable parameters will stop responding to variations in error, and will not get adjusted during optimization. This dying ReLU problem takes out the gradient information needed for CPL attacks. In  comparison, both Sigmoid and Tanh are differentiable bijective and can pass the gradient from layer to layer in an almost lossless manner. LeakyReLU sets a slightly inclined line for the negative part of ReLU to mitigate the issue of dying ReLU and thus is vulnerable to CPL attacks.



Motivated by the impact of activation function, we argue that any model components that discontinue the integrity and uniqueness of gradients can hamper CPL attacks. We observe from our experiments that an added dropout structure enables different gradient in every query, making $\nabla w^{\tau}_{att}(t)$ elusive and unable to converge to the uploaded gradients. By contrast, pooling cannot prevent CPL attacks since pooling layers do not have parameters.

\vspace{-0.6cm}

\subsubsection{Baseline v.s. Communication-efficient Protocols:} We have discussed the communication-efficient parameter update protocol using low-rank filers in Section~2.3. As more FL systems utilize a communication-efficient protocol to replace the baseline protocol, it is important to study the impact of using a communication efficient protocol on the performance of the CPL attacks, especially compared to the baseline client-to-server communication protocol. In this set of experiments, we measure the performance of CPL attacks under varying gradient compression percentage $\theta$, i.e.,  $\theta$ percentage of the gradient update will be discarded in this round of gradient upload. We employ the compression
method in~\cite{lin2017deep} as it provides a good trade-off between communication-efficiency and model training accuracy. It leverages sparse updates and sends only the important gradients, i.e., the gradients whose magnitude larger than a threshold, and further measures are taken to avoid losing information. Locally, the client will accumulate small gradients and only send them when the accumulation is large enough. Figure~\ref{fig:dgc} shows the visualization of the comparison on MNIST and CIFAR10. We observe that compared to baseline protocol with full gradient upload, using the communication efficient protocol with $\theta$ up to 40\%, the CPL attack remains to be effective at the maximum attack iterations of 26 with 100\% attack success rate for CIFAR10.

\vspace{-0.3cm}
\begin{figure}[ht]
\centerline{\includegraphics[scale=.35]{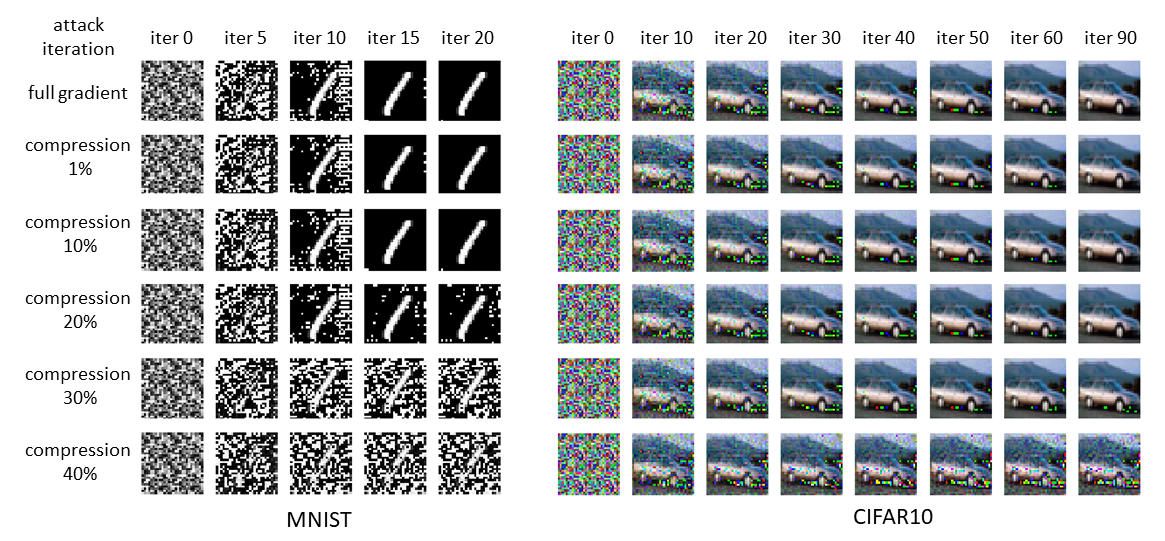}}
\vspace{-0.4cm}
\caption{\small Illustration of the CPL attack under communication-efficient update}
\label{fig:dgc}
\end{figure}


\subsection{Attack Effect and Cost Metrics}
\label{sec:3.3}

Our framework evaluates the adverse effect and cost of CPL attacks using the following metrics. For data-specific metrics, we average the evaluation results over all successful reconstructions.

\vspace{-0.4cm}

\subsubsection{Attack success rate (ASR)} is the percentage
of successfully reconstructed training data over the number of training data being attacked. We use ASRc and ASRl to refer to the attack success rate on content and label respectively.

\vspace{-0.4cm}

\subsubsection{MSE} uses the root mean square deviation to measure the similarity between reconstructed input $x_{rec}$ and ground-truth input $x$: $\frac{1}{M}\sum\nolimits_{i = 1}^M {(x(i) - {x_{rec}}(i)} {)^2}$ when the reconstruction is successful. $M$ denotes total number of features in the input. MSE can be used on all data format such as attributes and text. A smaller MSE means the more similar reconstructed data to the private ground truth.

\vspace{-0.4cm}

\subsubsection{SSIM} measures the structural similarity between two images based on a perception-based model~\cite{wang2004image} that considers image degradation as perceived change.
$$SSIM(x,x') = \frac{{(2{\mu _x}{\mu _{x'}} + {c_1})(2{\sigma _{xx'}} + {c_2})}}{{(\mu _x^2 + \mu _{x'}^2 + {c_1})(\sigma _x^2 + \sigma _{x'}^2 + {c_2})}},$$ where $\mu _{x}$ and $\mu_{x'}$ are
the average of $x$ and $x'$, $\sigma _{x}^{2}$ and $\sigma _{x'}^{2}$ are the variance of $x$ and $x'$. $\sigma _{xy}$ is the covariance of $x$ and $x'$. $c_{1}=(k_{1}L)^{2}$ and $c_{2}=(k_{2}L)^{2}$ are two variables to stabilize the division with weak denominator. $L$ is the dynamic range of the pixel-values.
$k_{1}=0.01$ and $k_{2}=0.03$ are constant by default. We use SSIM to evaluate all image datasets. The closer SSIM to 1, the better the attack quality in terms of image reconstruction. SSIM is designed to improve on traditional methods such as MSE on image similarity.

\vspace{-0.4cm}

\subsubsection{Attack iteration} measures the number of attack iterations required for reconstruction learning to converge and thus succeed the attack, e.g., $L_2$ distance of the gradients between the reconstructed data and the local private training data is smaller than a pre-set threshold.

\vspace{-0.4cm}

\section{Experiments and Results}

\label{experiments}

\subsection{Experiment Setup}

We evaluate CPL attacks on four image datasets: MNIST, LFW, CIFAR10, CIFAR100. MNIST consists of 70000 grey-scale hand-written digits images of size $28\times 28$. The 60000:10000 split is used for training and testing data. Labeled Faces in the Wild (LFW) people dataset has 13233 images from 5749 classes. The original image size is $250\times 250$ and we slide it to $32\times 32$ to extract the 'interesting' part. Our experiments only consider 106 classes, each with more than 14 images. For a total number of 3735 eligible LFW data, a 3:1 train-test ratio is applied. CIFAR10 and CIFAR100 each consists of 60000 color images of size $32\times 32$ with 10 classes and 100 classes respectively. The 50000:10000 split is used for training and testing.

We perform CPL attacks with the following attack configurations as the default unless otherwise stated. The initialization method is patterned, the maximum attack iteration is 300, the optimization method is L-BFGS with attack learning rate 1. The attack is performed with full gradient communication.
For each dataset, the attack is performed on 100 images with 10 different random seeds. For MNIST and LFW, we use a LeNet model with 0.9568 benign accuracy on MNIST and 0.695 on LFW. For CIFAR10 and CIFAR100, we apply a ResNet20 with benign accuracy of 0.863 on CIFAR10 and CIFAR100. We use 100 clients as the total client population and at each communication round, 10\% of clients will be selected randomly to participate in the federated learning.

\subsection{Gradient Leakage Attack Evaluation}

\label{sec4.2}

\subsubsection{Comparison with other gradient leakage attacks.}
We first conduct a set of experiments to compare the CPL-patterned attack with two existing gradient leakage attacks: the deep gradient attack~\cite{zhu2019deep}, and the gradient inverting attack~\cite{geiping2020inverting}, which replaces the $L_2$ distance function with cosine similarity and performs the optimization on the sign of the gradient. We first measure the attack results on the four benchmark image datasets. Table~\ref{table:base_evaluation} shows that CPL is a much faster and more efficient attack with the highest attack success rate (ASR) and lowest attack iterations on both content and label reconstruction for all four datasets. Also, the very high SSIM and low MSE for CPL indicate the quality of the reconstructed data is almost identical to the private training data. We also observe that gradient inverting attack~\cite{geiping2020inverting} can lead to high ASR compared to deep gradient attack~\cite{zhu2019deep} using $L_2$ distance but at a great cost of attack iterations. Note that CPL using $L_2$ distance offers slightly higher ASR compared to~\cite{geiping2020inverting} but at much lower attack cost in terms of attack iterations required.

\vspace{-0.5cm}

\begin{table}
\centering
\scalebox{0.75}{
\begin{tabular}{|c|c|c|c|c|c|c|c|c|c|c|c|c|}
\hline
\multirow{2}{*}{} & \multicolumn{3}{c|}{CIFAR10}                              & \multicolumn{3}{c|}{CIFAR100}                             & \multicolumn{3}{c|}{LFW}                        & \multicolumn{3}{c|}{MNIST}                                \\ \cline{2-13}
& CPL & \cite{zhu2019deep}   & \cite{geiping2020inverting} & CPL & \cite{zhu2019deep}   & \cite{geiping2020inverting} & CPL & \cite{zhu2019deep}   & \cite{geiping2020inverting} & CPL & \cite{zhu2019deep}   & \cite{geiping2020inverting} \\ \hline
attack iter & \textbf{28.3} & 114.5 & 6725 & \textbf{61.8} & 125 & 6813 & \textbf{25}  & 69.2 & 4527 & \textbf{11.5} & 18.4  & 3265 \\ \hline
ASRc & \textbf{0.973}    & 0.754    & 0.958                               & \textbf{0.981}    & 0.85     & 0.978                               & \textbf{1}        & 0.857    & 0.974                               & \textbf{1}        & 0.686    & 0.784                               \\ \hline
ASRl            & \textbf{1}        & 0.965    & \textbf{1}                                   &  \textbf{1}        & 0.94     & \textbf{1}                                   & \textbf{1}        & 0.951    & \textbf{1}                                   & \textbf{1}        & 0.951    & \textbf{1}                                   \\ \hline
SSIM              &   \textbf{ 0.9985}   & 0.9982   & 0.9984                              & \textbf{0.959}    & 0.953    & 0.958                               & \textbf{0.998}    & 0.997    & 0.9978                              & \textbf{0.99}     & 0.985    & 0.989                               \\ \hline
MSE                  & \textbf{2.2E-04} & 2.5E-04 & \textbf{2.2E-04}                            & \textbf{5.4E-04} & 6.5E-04 & \textbf{5.4E-04}                            & \textbf{2.2E-04} & 2.9E-04 & 2.3E-04                            & \textbf{1.5E-05} & 1.7E-05 & 1.6E-05                            \\ \hline
\end{tabular}
}
\caption{\small Comparison of Different Gradient Leakage Attacks}
\label{table:base_evaluation}
\end{table}

\vspace{-1.0cm}


In addition, we include two attribute datasets: UCI Adult Income and Breast Cancer Wisconsin in our comparison experiments. UCI Adult dataset includes 48842 records with 14 attributes such as age, gender, education,
marital status, occupation, working hours, and native country.
The  binary classification task is to predict if a person makes
over \$50K a year based on the census attributes. Breast Cancer Wisconsin has 569 records with 32 attributes that are computed from a digitized image of a fine needle aspirate (FNA) of a breast mass. These features describe characteristics of the cell nuclei present in the image. The task is to identify if the record indicates benign or malignant cancer. For two attributes dataset, a multi-layer-perceptron with one hidden layer is used and has an accuracy of 0.8646 on UCI Adult and 0.986 on Breast Cancer Wisconsin. Categorical features in the two attribute datasets are processed with one-hot encoding. Table~\ref{table:base_attribute} shows the results. For the breast cancer dataset, similar observation is found compared to Table~\ref{table:base_evaluation}. CPL is the most effective attack with the highest ASR, and the gradient inverting attack~\cite{geiping2020inverting} is the second with ASRc of 78\% and ASRl of 94\%, compared to the deep gradient attack~\cite{zhu2019deep} with ASRc of 35\% and ASRl of 56\%. For the UCI Adult dataset, all three have good and similar attack performance.

\vspace{-0.5cm}
\begin{table}
\centering
\scalebox{0.80}{
\begin{tabular}{|c|c|c|c|c|c|c|c|c|}
\cline{1-4} \cline{6-9}
UCI Adult & CPL      & \cite{zhu2019deep}  & \cite{geiping2020inverting} & \quad\quad\quad & Breast Cancer & CPL      & \cite{zhu2019deep}  & \cite{geiping2020inverting}  \\ \cline{1-4} \cline{6-9}
ASRc      & 1        & 0.99        &1     & \quad\quad\quad & ASRc          & 1        & 0.35     & 0.78     \\ \cline{1-4} \cline{6-9}
ASRl      & 1        & 1        & 1       & \quad\quad\quad & ASRl          & 1        & 0.56     & 0.94     \\ \cline{1-4} \cline{6-9}
MSE       & 1.82E-04 & 3.23E-04 & 4.97E-03 & \quad\quad\quad & MSE           & 4.61E-04 & 6.15E-04 & 6.29E-04 \\ \cline{1-4} \cline{6-9}
\end{tabular}
}
\caption{\small Comparison of gradient leakage attacks on two attribute datasets}
\label{table:base_attribute}
\end{table}

\vspace{-1.6cm}

\subsubsection{Comparison with other training data inference attack.}
This set of experiments compares client privacy leakage attacks with two existing training data inference attacks: Melis et al~\cite{melis2019exploiting} is based on membership-attack~\cite{shokri2017membership,truex2018towards} and can make partial property inference on a carefully selected subset of data. Aono et al~\cite{aono2017privacy} is an attack that only works on the first layer of a multi-layer perceptron and can synthesize data that is proportional to the private training data. Figure~\ref{fig:compare} shows the comparison result.As the gap of the MSE magnitude is too large, the results are plotted in the log scale. It shows that the adverse effect of the CPL attack is the most detrimental with the hightest SSIM and lowest MSE, followed by Aono et al as the second, and Melis et al has much smaller SSIM and much larger MSE, indicating the low attack inference quality.

\vspace{-0.5cm}
\begin{figure}[ht]
\begin{minipage}{0.49\linewidth}
 \centerline{\includegraphics[scale=.32]{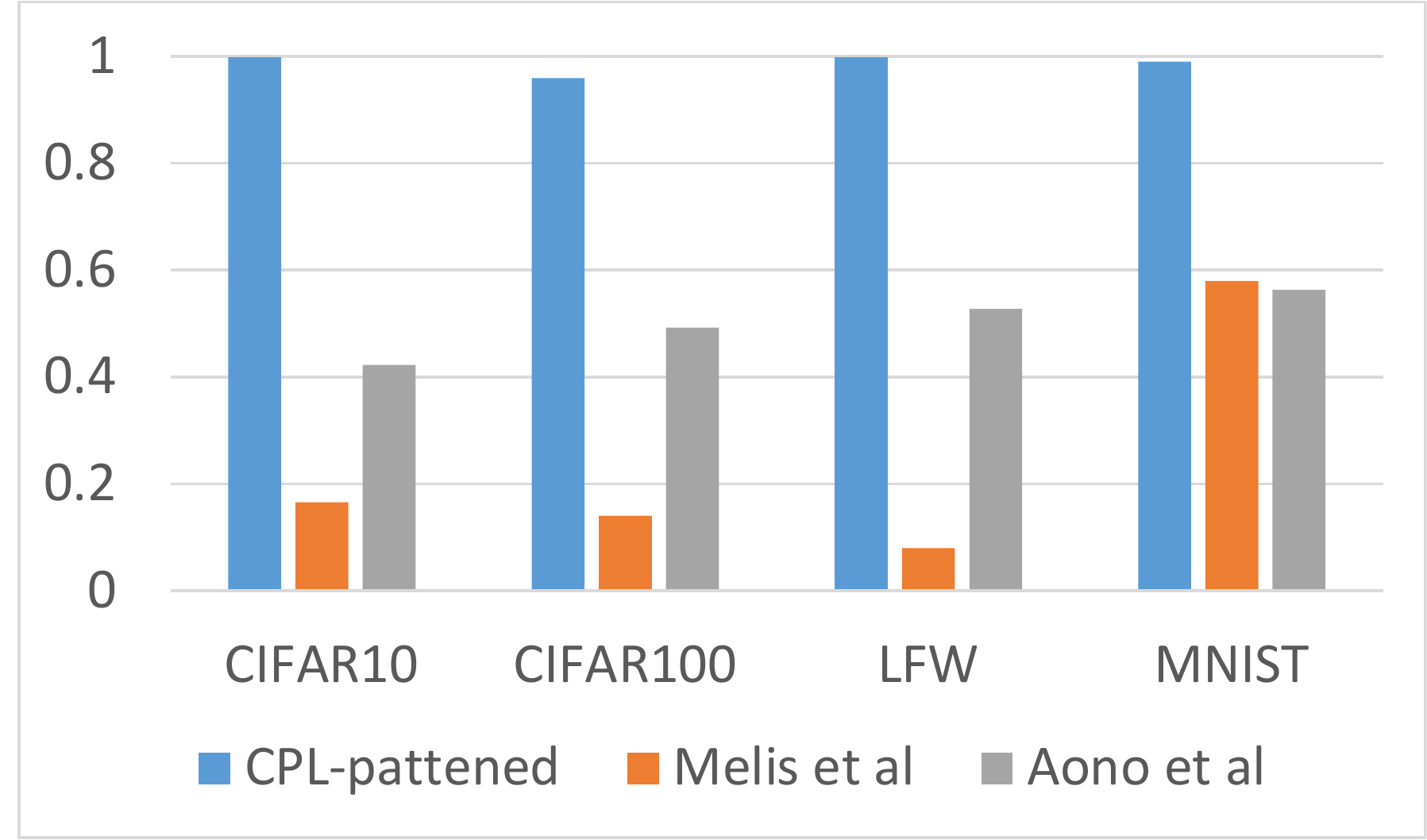}}
 \subcaption{\small SSIM}
 \label{fig:compare_ssim}
\end{minipage}
\begin{minipage}{0.49\linewidth}
 \centerline{\includegraphics[scale=.23]{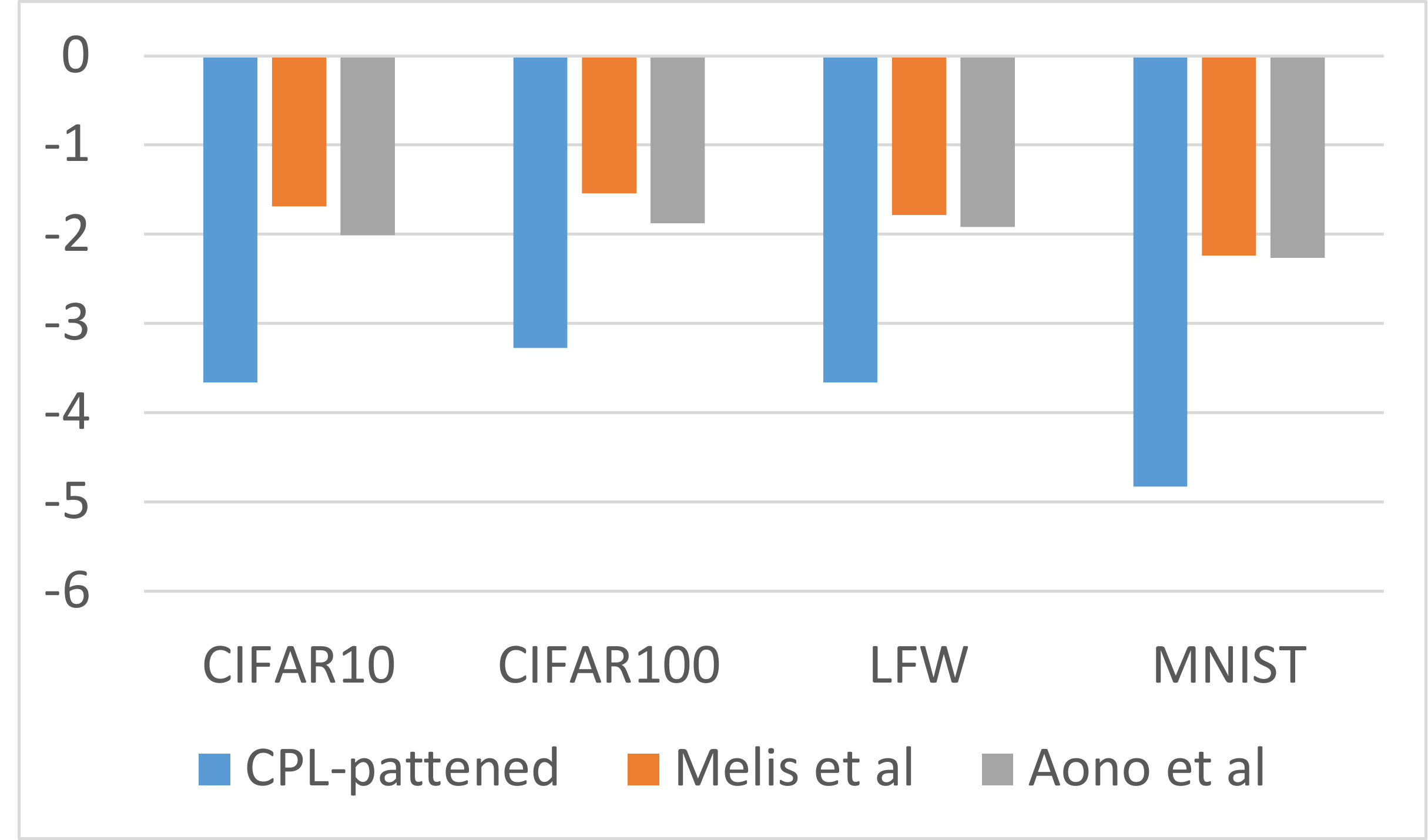}}
  \subcaption{\small MSE in log scale}
  \label{fig:compare_mse}
\end{minipage}
\caption{\small Comparison of CPL attack with Melis~\cite{melis2019exploiting} and Aono~\cite{aono2017privacy}}
\label{fig:compare}
\vspace{-0.4cm}
\end{figure}

\vspace{-0.8cm}

\subsubsection{Variation Study: Geometrical Initialization.}
This set of experiments measure and compare the four types of geometrical initialization methods: patterned random, dark/light, RGB, and optimal. For optimal initialization, we feed a piece of data that is randomly picked from the training set. This assumption is reasonable when different clients hold part of the information about one data item. Table~\ref{table:geometrical} shows the result.
We observe that the performance of all four geometrical initializations is always better than the random initialization. Note that the optimal initialization is used in this experiment as a theoretical optimal reference point as it assumes the background information about the data distribution. Furthermore, the performance of geometrical initializations  is also dataset-dependent. CPL attack on CIFAR100 requires a longer time and more iterations to succeed than CPL on CIFAR10 and LFW.


\vspace{-0.4cm}

\begin{table}[ht]
\centering
\scalebox{0.8}{
\small{
\begin{tabular}{|c|c|C{1.2cm}|C{0.9cm}|C{0.9cm}|C{0.9cm}|C{0.9cm}|C{0.9cm}|C{0.9cm}|C{0.9cm}|C{1.1cm}|}
\hline
\multicolumn{2}{|c|}{\multirow{2}{*}{\diagbox{dataset}{initialization}}} & baseline  & \multicolumn{2}{c|}{patterned} & \multicolumn{2}{c|}{dark/light} & \multicolumn{3}{c|}{RGB} & optimal \\ \cline{3-11}
\multicolumn{2}{|c|}{}                                & random & 2*2            & 4*4           & dark           & light          & R      & G      & B      & insider \\ \hline
\multirow{2}{*}{CIFAR10}         & attack iter       & 91.14  & 28.3           & \textbf{24.8}          & 34             & 52.3           & 35.9   & 77.5   & 79.1   & 23.2    \\ \cline{2-11}
                                  & ASR              & 0.871  & 0.973          & \textbf{0.976}         & 0.99           & 1              & 0.99   & 0.96   & 0.96   & 1       \\ \hline
\multirow{2}{*}{CIFAR100}        & attack iter       & 125    & 61.8           & 57.2          & \textbf{57.5}           & 65.3           & 59.4   & 61.3   & 62.4   & 35.3    \\ \cline{2-11}
                                  & ASR              & 1      & 0.981          & 0.995         & \textbf{1}              & 1              & 1      & 0.88   & 1      & 1       \\ \hline
\multirow{2}{*}{LFW}              & attack iter       & 71.1   & 25             & 18.6          & 34             & 50.8           & \textbf{20.3}   & 28     & 42.1   & 13.3    \\ \cline{2-11}
                                  & ASR              & 0.86   & 1              & 0.997         & 1              & 1              & \textbf{1}      & 1      & 1      & 1       \\ \hline
\end{tabular}
}}
\caption{\small Comparison of different geometrical initialization in CPL attacks}
\label{table:geometrical}
\vspace{-0.4cm}
\end{table}

\vspace{-0.4cm}

\noindent
\begin{minipage}{\textwidth}
  \begin{minipage}{0.54\linewidth}
 \centerline{\includegraphics[scale=.25]{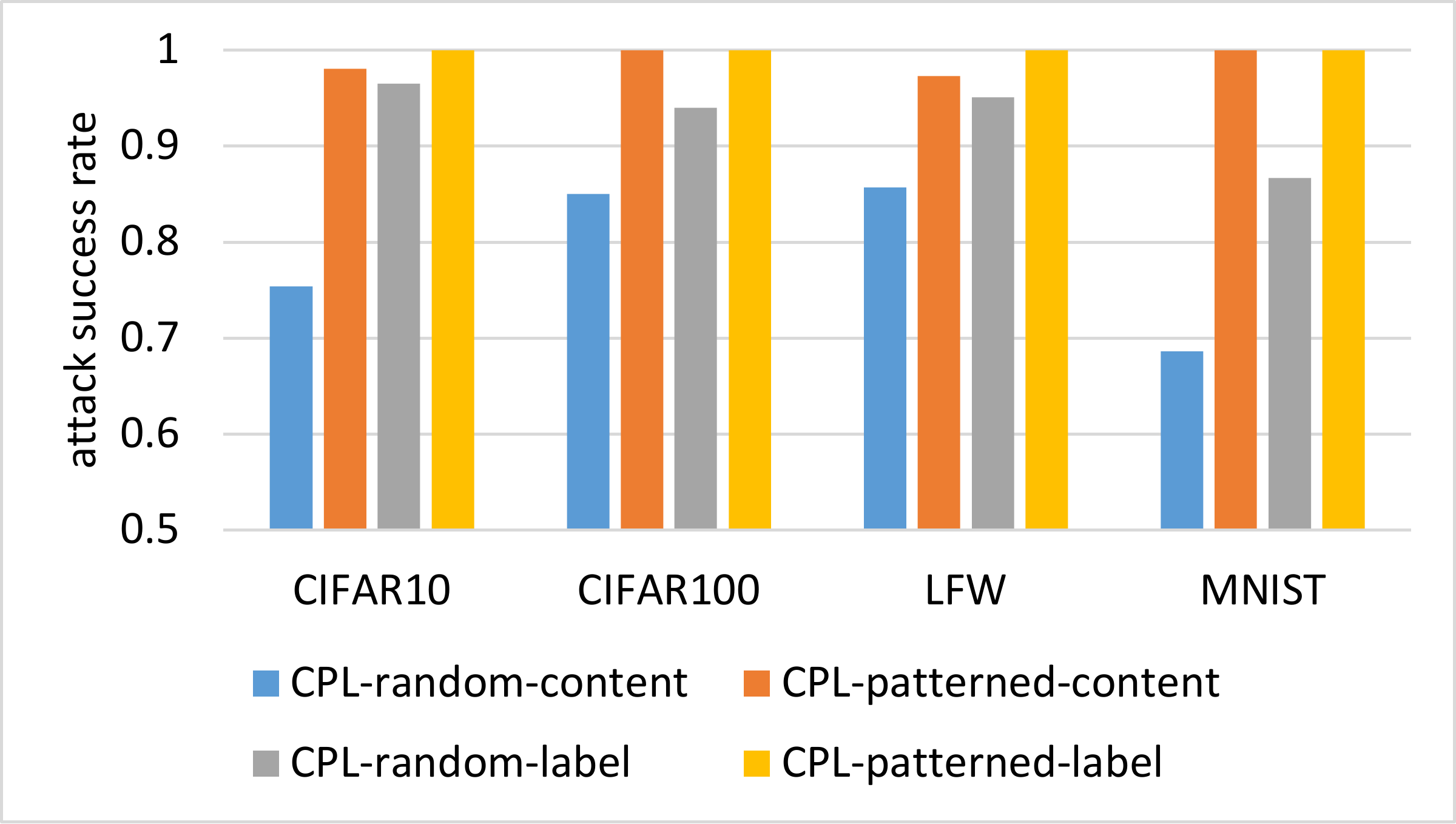}}
  \captionof{figure}{\small Effect of Attack Initialization}
  \label{fig:init_compare}
\end{minipage}
\begin{minipage}{0.45\linewidth}
\scalebox{0.70}{
\begin{tabular}{|c|c|c|c|c|c|}
\hline
\multicolumn{2}{|c|}{attack iter variance} & avg   & min & max & variance \\ \hline
\multirow{2}{*}{MNIST}     & CPL-patterned & 11.5  & 7   & 19  & 3.844    \\ \cline{2-6}
                           & CPL-random     & 18.4  & 13  & 65  & 12.27    \\ \hline
\multirow{2}{*}{LFW}       & CPL-patterned & 25    & 18  & 125 & 14.32    \\ \cline{2-6}
                           & CPL-random    & 69.2  & 53  & 186 & 25.87    \\ \hline
\multirow{2}{*}{CIFAR10}   & CPL-patterned & 28.3  & 17  & 86  & 18.11    \\ \cline{2-6}
                           & CPL-random    & 114.5 & 67  & 272 & 39.31    \\ \hline
\multirow{2}{*}{CIFAR100}  & CPL-patterned & 61.8  & 24  & 133 & 25.3     \\ \cline{2-6}
                           & CPL-random     & 125   & 70  & 240 & 41.2     \\ \hline
\end{tabular}
}
 \captionof{table}{{\footnotesize Variance of attack iterations}}
\label{table:att_variance}
\end{minipage}
\vspace{-0.4cm}
\end{minipage}

\begin{table}[ht]
\begin{minipage}{0.49\linewidth}
\scalebox{0.80}{
\begin{tabular}{|c|c|c|c|c|c|}
\hline
\begin{tabular}[c]{@{}c@{}}batch   size \end{tabular} & 1       & 2       & 4       & 8     & 16    \\ \hline
ASR                                           & 1       & 0.96    & 0.89    & 0.76  & 0.13  \\ \hline
attack iter                                            & 25      & 25.7    & 25.6    & 26.1  & 25.7  \\ \hline
SSIM                                                   & 0.998   & 0.635   & 0.525   & 0.456 & 0.401 \\ \hline
MSE                                                    & 2.2E-04 & 6.7E-03 & 8.0E-03 & 9.0E-03 & 1.0E-02  \\ \hline
\end{tabular}
}
 \subcaption{\small Effect of Batch size on CPL}
 \label{table:batchsize}
\end{minipage}
\hspace{0.2cm}
\begin{minipage}{0.49\linewidth}
\scalebox{0.80}{
\begin{tabular}{|c|c|c|c|c|c|}
\hline
local iter  & 1       & 3       & 5       & 7       & 9       \\ \hline
attack iter & 25      & 42.5    & 94.2    & 95.6    & 97.9    \\ \hline
ASR         & 1       & 1       & 0.97    & 0.85    & 0.39    \\ \hline
SSIM        & 0.998   & 0.981   & 0.898   & 0.659   & 0.471   \\ \hline
MSE         & 2.2E-04 & 6.8E-04 & 1.8E-03 & 3.8E-03 & 5.5E-03 \\ \hline
\end{tabular}
}
\subcaption{\small Effect of local training iterations}
\label{table:local_iter}
\end{minipage}
\caption{Effect of local training hyperparameters on CPL attack (LFW)}
\label{table:vulnerability}
\vspace{-0.4cm}
\end{table}

\vspace{-0.4cm}

Figure~\ref{fig:init_compare} compares the four geometric initialization methods with respect to the effectiveness of CPL attack on four benchmark datasets. Table~\ref{table:att_variance} measures the number of iterations to succeed the CPL attack in terms of average, min, max, and variance for each of the four datasets. We observe from Figure~\ref{fig:init_compare} that the geometrical initialization can largely increase the attack success rate in addition to faster attack convergence. Furthermore, the attack success rate on the label reconstruction is consistently higher than the attack success rate on content reconstruction, which further confirms the observation in Figure~\ref{fig:init_seed} (last row for LFW and CIFAR100). It indicates that separating the content reconstruction attack and label reconstruction attack during the CPL attack iterations, as done in CPL (Algorithm~1), is another factor contributing to high success rate and low cost (fewer attack iterations) compared to optimizing the content attack and label attack simultaneously~\cite{zhu2019deep} (recall Table~\ref{table:base_evaluation} and Table~\ref{table:base_attribute}).

\subsubsection{Variation Study: batch size and iterations.}
Motivated by the batch size visualization in~Figure~\ref{fig:batchsize}, we study the impact of hyperparameters used in local training, such as batch size and the number of local training iterations, on the performance of the CPL attack. Table~\ref{table:batchsize} shows the results of the CPL attack on the LFW dataset with five different batch sizes. We observe that the ASR of CPL attack is decreased to 96\%, 89\%, 76\%, and 13\% as the batch size increases to 2, 4, 8, and 16. The CPL attacks at different batch sizes are successful at the attack iterations around 25 to 26 as we measure attack iterations only on successfully reconstructed instances. Table~\ref{table:local_iter} shows the results of the CPL attack under five different settings of local iterations before sharing the gradient updates. We show that as more iterations are performed at local training before sharing the gradient update, the ASR of the CPL attack is decreasing with 97\%, 85\%, and 39\% for iterations of 5, 7, and 9 respectively. This inspires us to add hyperparameter optimization as one of the mitigation strategies for evaluation in Section~\ref{sec4.3}.

\vspace{-0.4cm}

\begin{table}[ht]
\begin{minipage}{0.99\linewidth}
\centering
\scalebox{0.80}{
\small{
\begin{tabular}{cccccccccccc}
\hline
\multicolumn{2}{|c|}{benign acc} & \multicolumn{1}{c|}{0}       & \multicolumn{1}{c|}{1\%}      & \multicolumn{1}{c|}{10\%}     & \multicolumn{1}{c|}{20\%}     & \multicolumn{1}{c|}{30\%}     & \multicolumn{1}{c|}{40\%}     & \multicolumn{1}{c|}{50\%}     & \multicolumn{1}{c|}{70\%}     & \multicolumn{1}{c|}{80\%}     & \multicolumn{1}{c|}{90\%}     \\ \hline
  \multicolumn{2}{|c|}{LFW}         & \multicolumn{1}{c|}{0.695}   & \multicolumn{1}{c|}{0.697}   & \multicolumn{1}{c|}{0.705}   & \multicolumn{1}{c|}{0.701}   & \multicolumn{1}{c|}{0.71}    & \multicolumn{1}{c|}{0.709}   & \multicolumn{1}{c|}{0.713}   & \multicolumn{1}{c|}{0.711}   & \multicolumn{1}{c|}{0.683}   & \multicolumn{1}{c|}{0.676}   \\ \hline
 \multicolumn{2}{|c|}{CIFAR100}    & \multicolumn{1}{c|}{0.67}    & \multicolumn{1}{c|}{0.673}   & \multicolumn{1}{c|}{0.679}   & \multicolumn{1}{c|}{0.685}   & \multicolumn{1}{c|}{0.687}   & \multicolumn{1}{c|}{0.695}   & \multicolumn{1}{c|}{0.689}   & \multicolumn{1}{c|}{0.694}   & \multicolumn{1}{c|}{0.676}   & \multicolumn{1}{c|}{0.668}   \\ \hline
 \multicolumn{2}{|c|}{CIFAR10}     & \multicolumn{1}{c|}{0.863}   & \multicolumn{1}{c|}{0.864}   & \multicolumn{1}{c|}{0.867}   & \multicolumn{1}{c|}{0.872}   & \multicolumn{1}{c|}{0.868}   & \multicolumn{1}{c|}{0.865}   & \multicolumn{1}{c|}{0.868}   & \multicolumn{1}{c|}{0.861}   & \multicolumn{1}{c|}{0.864}   & \multicolumn{1}{c|}{0.859}   \\ \hline
\multicolumn{2}{|c|}{MNIST}       & \multicolumn{1}{c|}{0.9568}  & \multicolumn{1}{c|}{0.9567}  & \multicolumn{1}{c|}{0.9577}  & \multicolumn{1}{c|}{0.957}   & \multicolumn{1}{c|}{0.9571}  & \multicolumn{1}{c|}{0.9575}  & \multicolumn{1}{c|}{0.9572}  & \multicolumn{1}{c|}{0.9576}  & \multicolumn{1}{c|}{0.9573}  & \multicolumn{1}{c|}{0.9556}  \\ \hline
\end{tabular}
}}
\subcaption{\small Benign accuracy of four datasets with varying compression rates}
\label{table:dgc_acc}
\end{minipage}
\hspace{0.5cm}
\begin{minipage}{0.99\linewidth}
\centering
\scalebox{0.80}{
\small{
\begin{tabular}{cccccccccccc}
\hline
\multicolumn{1}{|c|}{\multirow{5}{*}{\rotatebox{90}{LFW}}}      & \multicolumn{1}{c|}{compression}            & \multicolumn{1}{c|}{original} & \multicolumn{1}{c|}{1\%}      & \multicolumn{1}{c|}{10\%}     & \multicolumn{1}{c|}{20\%}     & \multicolumn{1}{c|}{30\%}     & \multicolumn{1}{c|}{40\%}     & \multicolumn{1}{c|}{50\%}     & \multicolumn{1}{c|}{70\%}     & \multicolumn{1}{c|}{80\%}     & \multicolumn{1}{c|}{90\%}     \\ \cline{2-12}
\multicolumn{1}{|c|}{}                          & \multicolumn{1}{c|}{attack iter} & \multicolumn{1}{c|}{25}       & \multicolumn{1}{c|}{25}      & \multicolumn{1}{c|}{24.9}    & \multicolumn{1}{c|}{24.9}    & \multicolumn{1}{c|}{25}      & \multicolumn{1}{c|}{24.8}    & \multicolumn{1}{c|}{25}      & \multicolumn{1}{c|}{24.6}    & \multicolumn{1}{c|}{24.5}    & \multicolumn{1}{c|}{300}     \\ \cline{2-12}
\multicolumn{1}{|c|}{}                          & \multicolumn{1}{c|}{ASR}         & \multicolumn{1}{c|}{1}        & \multicolumn{1}{c|}{1}       & \multicolumn{1}{c|}{1}       & \multicolumn{1}{c|}{1}       & \multicolumn{1}{c|}{1}       & \multicolumn{1}{c|}{1}       & \multicolumn{1}{c|}{1}       & \multicolumn{1}{c|}{1}       & \multicolumn{1}{c|}{1}       & \multicolumn{1}{c|}{0}       \\ \cline{2-12}
\multicolumn{1}{|c|}{}                          & \multicolumn{1}{c|}{SSIM}        & \multicolumn{1}{c|}{0.998}    & \multicolumn{1}{c|}{0.9996}  & \multicolumn{1}{c|}{0.9997}  & \multicolumn{1}{c|}{0.9978}  & \multicolumn{1}{c|}{0.9978}  & \multicolumn{1}{c|}{0.9975}  & \multicolumn{1}{c|}{0.998}   & \multicolumn{1}{c|}{0.9981}  & \multicolumn{1}{c|}{0.951}   & \multicolumn{1}{c|}{0.004}   \\ \cline{2-12}
\multicolumn{1}{|c|}{}                          & \multicolumn{1}{c|}{MSE}         & \multicolumn{1}{c|}{2.2E-04}  & \multicolumn{1}{c|}{1.8E-04} & \multicolumn{1}{c|}{1.7E-04} & \multicolumn{1}{c|}{4.9E-04} & \multicolumn{1}{c|}{4.8E-04} & \multicolumn{1}{c|}{5.1E-04} & \multicolumn{1}{c|}{4.5E-04} & \multicolumn{1}{c|}{4.6E-04} & \multicolumn{1}{c|}{1.6E-03} & \multicolumn{1}{c|}{1.6E-01} \\ \hline
\multicolumn{1}{l}{}                            & \multicolumn{1}{l}{}             & \multicolumn{1}{l}{}          & \multicolumn{1}{l}{}         & \multicolumn{1}{l}{}         & \multicolumn{1}{l}{}         & \multicolumn{1}{l}{}         & \multicolumn{1}{l}{}         & \multicolumn{1}{l}{}         & \multicolumn{1}{l}{}         & \multicolumn{1}{l}{}         & \multicolumn{1}{l}{}         \\ \hline
\multicolumn{1}{|c|}{\multirow{4}{*}{\rotatebox{90}{CIFAR100}}} & \multicolumn{1}{c|}{attack iter} & \multicolumn{1}{c|}{61.8}     & \multicolumn{1}{c|}{61.8}    & \multicolumn{1}{c|}{61.8}    & \multicolumn{1}{c|}{61.7}    & \multicolumn{1}{c|}{61.7}    & \multicolumn{1}{c|}{61.5}    & \multicolumn{1}{c|}{61.8}    & \multicolumn{1}{c|}{60.1}    & \multicolumn{1}{c|}{59.8}    & \multicolumn{1}{c|}{300}     \\ \cline{2-12}
\multicolumn{1}{|c|}{}                          & \multicolumn{1}{c|}{ASR}         & \multicolumn{1}{c|}{1}        & \multicolumn{1}{c|}{1}       & \multicolumn{1}{c|}{1}       & \multicolumn{1}{c|}{1}       & \multicolumn{1}{c|}{1}       & \multicolumn{1}{c|}{1}       & \multicolumn{1}{c|}{1}       & \multicolumn{1}{c|}{1}       & \multicolumn{1}{c|}{1}       & \multicolumn{1}{c|}{0}       \\ \cline{2-12}
\multicolumn{1}{|c|}{}                          & \multicolumn{1}{c|}{SSIM}        & \multicolumn{1}{c|}{0.959}    & \multicolumn{1}{c|}{0.9994}  & \multicolumn{1}{c|}{0.9981}  & \multicolumn{1}{c|}{0.9981}  & \multicolumn{1}{c|}{0.998}   & \multicolumn{1}{c|}{0.9983}  & \multicolumn{1}{c|}{0.9982}  & \multicolumn{1}{c|}{0.9983}  & \multicolumn{1}{c|}{0.895}   & \multicolumn{1}{c|}{0.016}   \\ \cline{2-12}
\multicolumn{1}{|c|}{}                          & \multicolumn{1}{c|}{MSE}         & \multicolumn{1}{c|}{5.4E-04}  & \multicolumn{1}{c|}{3.3E-04} & \multicolumn{1}{c|}{3.7E-04} & \multicolumn{1}{c|}{3.7E-04} & \multicolumn{1}{c|}{3.8E-04} & \multicolumn{1}{c|}{3.5E-04} & \multicolumn{1}{c|}{3.6E-04} & \multicolumn{1}{c|}{3.7E-04} & \multicolumn{1}{c|}{1.5E-03} & \multicolumn{1}{c|}{1.2E-01} \\ \hline
\multicolumn{1}{l}{}                            & \multicolumn{1}{l}{}             & \multicolumn{1}{l}{}          & \multicolumn{1}{l}{}         & \multicolumn{1}{l}{}         & \multicolumn{1}{l}{}         & \multicolumn{1}{l}{}         & \multicolumn{1}{l}{}         & \multicolumn{1}{l}{}         & \multicolumn{1}{l}{}         & \multicolumn{1}{l}{}         & \multicolumn{1}{l}{}         \\ \hline
\multicolumn{1}{|c|}{\multirow{4}{*}{\rotatebox{90}{CIFAR10}}}  & \multicolumn{1}{c|}{attack iter} & \multicolumn{1}{c|}{28.3}     & \multicolumn{1}{c|}{28.3}    & \multicolumn{1}{c|}{28.1}    & \multicolumn{1}{c|}{26.5}    & \multicolumn{1}{c|}{25.8}    & \multicolumn{1}{c|}{25.3}    & \multicolumn{1}{c|}{300}     & \multicolumn{1}{c|}{300}     & \multicolumn{1}{c|}{300}     & \multicolumn{1}{c|}{300}     \\ \cline{2-12}
\multicolumn{1}{|c|}{}                          & \multicolumn{1}{c|}{ASR}         & \multicolumn{1}{c|}{1}        & \multicolumn{1}{c|}{1}       & \multicolumn{1}{c|}{1}       & \multicolumn{1}{c|}{1}       & \multicolumn{1}{c|}{1}       & \multicolumn{1}{c|}{1}       & \multicolumn{1}{c|}{0}       & \multicolumn{1}{c|}{0}       & \multicolumn{1}{c|}{0}       & \multicolumn{1}{c|}{0}       \\ \cline{2-12}
\multicolumn{1}{|c|}{}                          & \multicolumn{1}{c|}{SSIM}        & \multicolumn{1}{c|}{0.9985}   & \multicolumn{1}{c|}{0.9996}  & \multicolumn{1}{c|}{0.9996}  & \multicolumn{1}{c|}{0.9997}  & \multicolumn{1}{c|}{0.9992}  & \multicolumn{1}{c|}{0.87}    & \multicolumn{1}{c|}{0.523}   & \multicolumn{1}{c|}{0.0017}  & \multicolumn{1}{c|}{0.0019}  & \multicolumn{1}{c|}{0.0018}  \\ \cline{2-12}
\multicolumn{1}{|c|}{}                          & \multicolumn{1}{c|}{MSE}         & \multicolumn{1}{c|}{2.2E-04}  & \multicolumn{1}{c|}{1.3E-04} & \multicolumn{1}{c|}{1.2E-04} & \multicolumn{1}{c|}{1.2E-04} & \multicolumn{1}{c|}{2.1E-04} & \multicolumn{1}{c|}{3.1E-03} & \multicolumn{1}{c|}{9.6E-03} & \multicolumn{1}{c|}{3.3E-01} & \multicolumn{1}{c|}{3.3E-01} & \multicolumn{1}{c|}{3.3E-01} \\ \hline
\multicolumn{1}{l}{}                            & \multicolumn{1}{l}{}             & \multicolumn{1}{l}{}          & \multicolumn{1}{l}{}         & \multicolumn{1}{l}{}         & \multicolumn{1}{l}{}         & \multicolumn{1}{l}{}         & \multicolumn{1}{l}{}         & \multicolumn{1}{l}{}         & \multicolumn{1}{l}{}         & \multicolumn{1}{l}{}         & \multicolumn{1}{l}{}         \\ \hline
\multicolumn{1}{|c|}{\multirow{4}{*}{\rotatebox{90}{MNIST}}}    & \multicolumn{1}{c|}{attack iter} & \multicolumn{1}{c|}{11.5}     & \multicolumn{1}{c|}{11.5}    & \multicolumn{1}{c|}{11.2}    & \multicolumn{1}{c|}{10.7}    & \multicolumn{1}{c|}{7.2}     & \multicolumn{1}{c|}{300}     & \multicolumn{1}{c|}{300}     & \multicolumn{1}{c|}{300}     & \multicolumn{1}{c|}{300}     & \multicolumn{1}{c|}{300}     \\ \cline{2-12}
\multicolumn{1}{|c|}{}                          & \multicolumn{1}{c|}{ASR}         & \multicolumn{1}{c|}{1}        & \multicolumn{1}{c|}{1}       & \multicolumn{1}{c|}{1}       & \multicolumn{1}{c|}{1}       & \multicolumn{1}{c|}{1}       & \multicolumn{1}{c|}{0}       & \multicolumn{1}{c|}{0}       & \multicolumn{1}{c|}{0}       & \multicolumn{1}{c|}{0}       & \multicolumn{1}{c|}{0}       \\ \cline{2-12}
\multicolumn{1}{|c|}{}                          & \multicolumn{1}{c|}{SSIM}        & \multicolumn{1}{c|}{0.99}     & \multicolumn{1}{c|}{0.9899}  & \multicolumn{1}{c|}{0.9891}  & \multicolumn{1}{c|}{0.9563}  & \multicolumn{1}{c|}{0.9289}  & \multicolumn{1}{c|}{0.8889}  & \multicolumn{1}{c|}{0.8137}  & \multicolumn{1}{c|}{0.425}   & \multicolumn{1}{c|}{0.433}   & \multicolumn{1}{c|}{0.43}    \\ \cline{2-12}
\multicolumn{1}{|c|}{}                          & \multicolumn{1}{c|}{MSE}         & \multicolumn{1}{c|}{2.4E-04}  & \multicolumn{1}{c|}{2.4E-04} & \multicolumn{1}{c|}{2.2E-04} & \multicolumn{1}{c|}{1.7E-03} & \multicolumn{1}{c|}{8.8E-03} & \multicolumn{1}{c|}{2.8E-02} & \multicolumn{1}{c|}{5.8E-02} & \multicolumn{1}{c|}{2.7E-01} & \multicolumn{1}{c|}{2.7E-01} & \multicolumn{1}{c|}{2.7E-01} \\ \hline
\end{tabular}
}}
\subcaption{\small Attack performance of four datasets with varying compression rates}
\label{table:dgc_asr}
\end{minipage}
\caption{\small Effect of CPL attack under communication-efficient FL protocols}
\label{table:dgc}
\vspace{-0.4cm}
\end{table}


\vspace{-1.1cm}


\subsubsection{Variation Study: Leakage in communication efficient sharing.} This set of experiments measures and compares the gradient leakage in CPL under baseline protocol (full gradient sharing) and communication-efficient protocol (significant gradient sharing with low-rank filer). Table~\ref{table:dgc} shows the result. To illustrate the comparison results, we provide the accuracy of the baseline protocol and the communication-efficient protocol of varying compression percentages on all four benchmark datasets in Table~\ref{table:dgc}(a). We make two interesting observations. (1) CPL attack can generate high confidence reconstructions (high ASR, high SSIM, low MSE) for MNIST and CIFAR10 at compression rate 40\%, and for CIFAR100 and LFW at the compression rate of 90\%. Second, as the compression percentage increases, the number of attack iterations to succeed the CPL attack decreases. This is because a larger portion of the gradients are low significance and are set to 0 by compression. When the attack fails, it indicates that the reconstruction cannot be done even with the infinite($\infty$) attack iterations, but we measure SSIM and MSE of the failed attacks at the maximum attack iterations of 300. (2) CPL attacks are more severe with more training labels in the federated learning task. A possible explanation is that informative gradients are more concentrated when there are more classes.



\subsection{Mitigation Strategies}
\label{sec4.3}

Motivated by our comprehensive analysis of CPL attacks of different forms, we next evaluate two attack mitigation strategies: gradient perturbation and gradient squeezing.

\noindent \textbf{Gradient Perturbation with Additive Noise.} We consider Gaussian noise and Laplace noise with zero means and different magnitude of variance in this set of experiments. Table~\ref{table:noise} provides the mitigation results on CIFAR100 and LFW. In both cases, the client privacy leakage attack is largely mitigated at some cost of accuracy if we add sufficient Gaussian noise (G-10e-2) or Laplace noise (L-10e-2), which causes small SSIM and large MSE, showing poor quality of reconstruction attack. Visualization of two examples from each dataset is given in Figure~\ref{fig:noise_large}.

\vspace{-0.5cm}

\begin{table}[ht]
\centering
\scalebox{0.8}{
\small{
\begin{tabular}{ccccccccc}
\hline
\multicolumn{1}{|c|}{}               & \multicolumn{4}{c|}{CIFAR100}                                                                                              & \multicolumn{4}{c|}{LFW}                                                                                                   \\ \hline
\multicolumn{1}{|c|}{Gaussian noise} & \multicolumn{1}{c|}{original} & \multicolumn{1}{c|}{G-10e-4} & \multicolumn{1}{c|}{G-10e-3} & \multicolumn{1}{c|}{G-10e-2} & \multicolumn{1}{c|}{original} & \multicolumn{1}{c|}{G-10e-4} & \multicolumn{1}{c|}{G-10e-3} & \multicolumn{1}{c|}{G-10e-2} \\ \hline
\multicolumn{1}{|c|}{benign acc}     & \multicolumn{1}{c|}{0.67}     & \multicolumn{1}{c|}{0.664}   & \multicolumn{1}{c|}{0.647}   & \multicolumn{1}{c|}{0.612}   & \multicolumn{1}{c|}{0.695}    & \multicolumn{1}{c|}{0.692}   & \multicolumn{1}{c|}{0.653}   & \multicolumn{1}{c|}{0.636}   \\ \hline
\multicolumn{1}{|c|}{attack iter}    & \multicolumn{1}{c|}{61.8}     & \multicolumn{1}{c|}{61.8}    & \multicolumn{1}{c|}{61.8}    & \multicolumn{1}{c|}{300}     & \multicolumn{1}{c|}{25}       & \multicolumn{1}{c|}{25}      & \multicolumn{1}{c|}{25}      & \multicolumn{1}{c|}{300}     \\ \hline
\multicolumn{1}{|c|}{ASR}           & \multicolumn{1}{c|}{1}        & \multicolumn{1}{c|}{1}       & \multicolumn{1}{c|}{1}       & \multicolumn{1}{c|}{0}       & \multicolumn{1}{c|}{1}        & \multicolumn{1}{c|}{1}       & \multicolumn{1}{c|}{1}       & \multicolumn{1}{c|}{0}       \\ \hline
\multicolumn{1}{|c|}{SSIM}           & \multicolumn{1}{c|}{0.9995}   & \multicolumn{1}{c|}{0.9976}  & \multicolumn{1}{c|}{0.8612}  & \multicolumn{1}{c|}{0.019}   & \multicolumn{1}{c|}{0.998}    & \multicolumn{1}{c|}{0.9976}  & \multicolumn{1}{c|}{0.8645}  & \multicolumn{1}{c|}{0.013}   \\ \hline
\multicolumn{1}{|c|}{MSE}            & \multicolumn{1}{c|}{5.4E-04}  & \multicolumn{1}{c|}{6.9E-04} & \multicolumn{1}{c|}{4.1E-03} & \multicolumn{1}{c|}{3.0E-01} & \multicolumn{1}{c|}{2.2E-04}  & \multicolumn{1}{c|}{3.7E-04} & \multicolumn{1}{c|}{3.0E-03} & \multicolumn{1}{c|}{1.9E-01} \\ \hline
                                     &                               &                              &                              &                              &                               &                              &                              &                              \\ \hline
\multicolumn{1}{|c|}{Laplace noise} & \multicolumn{1}{c|}{original} & \multicolumn{1}{c|}{L-10e-4} & \multicolumn{1}{c|}{L-10e-3} & \multicolumn{1}{c|}{L-10e-2} & \multicolumn{1}{c|}{original} & \multicolumn{1}{c|}{L-10e-4} & \multicolumn{1}{c|}{L-10e-3} & \multicolumn{1}{c|}{L-10e-2} \\ \hline
\multicolumn{1}{|c|}{benign acc}     & \multicolumn{1}{c|}{0.67}     & \multicolumn{1}{c|}{0.651}   & \multicolumn{1}{c|}{0.609}   & \multicolumn{1}{c|}{0.578}   & \multicolumn{1}{c|}{0.695}    & \multicolumn{1}{c|}{0.683}   & \multicolumn{1}{c|}{0.632}   & \multicolumn{1}{c|}{0.597}   \\ \hline
\multicolumn{1}{|c|}{attack iter}    & \multicolumn{1}{c|}{61.8}     & \multicolumn{1}{c|}{61.8}    & \multicolumn{1}{c|}{61.8}    & \multicolumn{1}{c|}{300}     & \multicolumn{1}{c|}{25}       & \multicolumn{1}{c|}{25}      & \multicolumn{1}{c|}{25}      & \multicolumn{1}{c|}{300}     \\ \hline
\multicolumn{1}{|c|}{ASR}           & \multicolumn{1}{c|}{1}        & \multicolumn{1}{c|}{1}       & \multicolumn{1}{c|}{1}       & \multicolumn{1}{c|}{0}       & \multicolumn{1}{c|}{1}        & \multicolumn{1}{c|}{1}       & \multicolumn{1}{c|}{1}       & \multicolumn{1}{c|}{0}       \\ \hline
\multicolumn{1}{|c|}{SSIM}           & \multicolumn{1}{c|}{0.9995}   & \multicolumn{1}{c|}{0.9956}  & \multicolumn{1}{c|}{0.7309}  & \multicolumn{1}{c|}{0.017}   & \multicolumn{1}{c|}{0.998}    & \multicolumn{1}{c|}{0.9965}  & \multicolumn{1}{c|}{0.803}   & \multicolumn{1}{c|}{0.009}   \\ \hline
\multicolumn{1}{|c|}{MSE}            & \multicolumn{1}{c|}{5.4E-04}  & \multicolumn{1}{c|}{6.4E-04} & \multicolumn{1}{c|}{6.4E-03} & \multicolumn{1}{c|}{3.1E-01} & \multicolumn{1}{c|}{2.2E-04}  & \multicolumn{1}{c|}{4.0E-04} & \multicolumn{1}{c|}{3.9E-03} & \multicolumn{1}{c|}{2.0E-01} \\ \hline
\end{tabular}
}}
\caption{\small Mitigation with Gaussian noise and Laplace noise }
\label{table:noise}
\vspace{-0.4cm}
\end{table}


\vspace{-0.6cm}

\begin{figure}[ht]
 \centerline{\includegraphics[scale=.37]{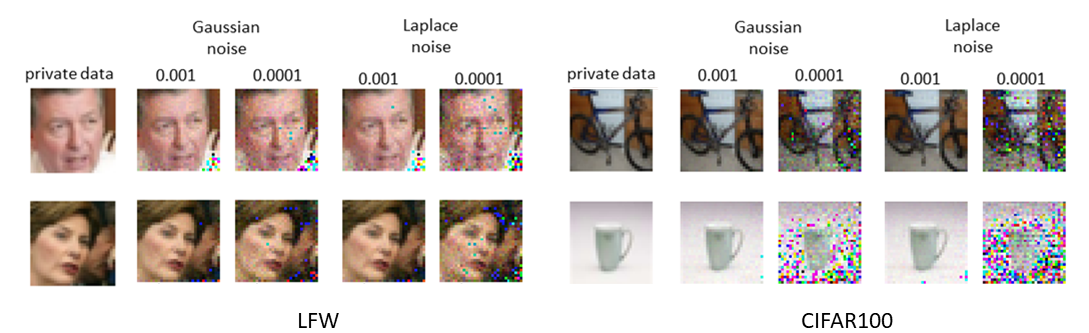}}
  \caption{\small Effect of additive noise on CPL-patterned attacks}
  \label{fig:noise_large}
\vspace{-0.4cm}
\end{figure}

\begin{table}[ht]
\centering
\scalebox{0.8}{
\small{
\begin{tabular}{cccccccccccc}
\hline
\multicolumn{1}{|c|}{\multirow{5}{*}{\rotatebox{90}{CIFAR10}}}  & \multicolumn{1}{c|}{local iter}  & \multicolumn{1}{c|}{1}       & \multicolumn{1}{c|}{2}       & \multicolumn{1}{c|}{3}       & \multicolumn{1}{c|}{4}       & \multicolumn{1}{c|}{5}       & \multicolumn{1}{c|}{6}       & \multicolumn{1}{c|}{7}       & \multicolumn{1}{c|}{8}       & \multicolumn{1}{c|}{9}       & \multicolumn{1}{c|}{10}      \\ \cline{2-12}
\multicolumn{1}{|c|}{}                           & \multicolumn{1}{c|}{ASR}         & \multicolumn{1}{c|}{0.973}   & \multicolumn{1}{c|}{0.971}   & \multicolumn{1}{c|}{0.921}   & \multicolumn{1}{c|}{0.875}   & \multicolumn{1}{c|}{0.835}   & \multicolumn{1}{c|}{0.758}   & \multicolumn{1}{c|}{0.612}   & \multicolumn{1}{c|}{0.5}     & \multicolumn{1}{c|}{0.406}   & \multicolumn{1}{c|}{0.292}   \\ \cline{2-12}
\multicolumn{1}{|c|}{}                           & \multicolumn{1}{c|}{SSIM}        & \multicolumn{1}{c|}{0.9985}  & \multicolumn{1}{c|}{0.9981}  & \multicolumn{1}{c|}{0.997}   & \multicolumn{1}{c|}{0.956}   & \multicolumn{1}{c|}{0.915}   & \multicolumn{1}{c|}{0.901}   & \multicolumn{1}{c|}{0.893}   & \multicolumn{1}{c|}{0.822}   & \multicolumn{1}{c|}{0.748}   & \multicolumn{1}{c|}{0.715}   \\ \cline{2-12}
\multicolumn{1}{|c|}{}                           & \multicolumn{1}{c|}{MSE}         & \multicolumn{1}{c|}{2.2E-04} & \multicolumn{1}{c|}{2.5E-04} & \multicolumn{1}{c|}{2.9E-04} & \multicolumn{1}{c|}{1.1E-03} & \multicolumn{1}{c|}{2.4E-03} & \multicolumn{1}{c|}{2.5E-03} & \multicolumn{1}{c|}{2.7E-03} & \multicolumn{1}{c|}{3.0E-03} & \multicolumn{1}{c|}{4.5E-03} & \multicolumn{1}{c|}{5.0E-03} \\ \cline{2-12}
\multicolumn{1}{|c|}{}                           & \multicolumn{1}{c|}{attack iter} & \multicolumn{1}{c|}{28.3}    & \multicolumn{1}{c|}{29.5}    & \multicolumn{1}{c|}{31.6}    & \multicolumn{1}{c|}{35.2}    & \multicolumn{1}{c|}{42.5}    & \multicolumn{1}{c|}{71.5}    & \multicolumn{1}{c|}{115.3}   & \multicolumn{1}{c|}{116.3}   & \multicolumn{1}{c|}{117.2}   & \multicolumn{1}{c|}{117.5}   \\ \hline
\multicolumn{1}{l}{}                             & \multicolumn{1}{l}{}             & \multicolumn{1}{l}{}         & \multicolumn{1}{l}{}         & \multicolumn{1}{l}{}         & \multicolumn{1}{l}{}         & \multicolumn{1}{l}{}         & \multicolumn{1}{l}{}         & \multicolumn{1}{l}{}         & \multicolumn{1}{l}{}         & \multicolumn{1}{l}{}         & \multicolumn{1}{l}{}         \\ \hline
\multicolumn{1}{|c|}{\multirow{4}{*}{\rotatebox{90}{CIFAR100}}} & \multicolumn{1}{c|}{ASR}         & \multicolumn{1}{c|}{0.981}   & \multicolumn{1}{c|}{0.977}   & \multicolumn{1}{c|}{0.958}   & \multicolumn{1}{c|}{0.949}   & \multicolumn{1}{c|}{0.933}   & \multicolumn{1}{c|}{0.893}   & \multicolumn{1}{c|}{0.842}   & \multicolumn{1}{c|}{0.78}    & \multicolumn{1}{c|}{0.557}   & \multicolumn{1}{c|}{0.437}   \\ \cline{2-12}
\multicolumn{1}{|c|}{}                           & \multicolumn{1}{c|}{SSIM}        & \multicolumn{1}{c|}{0.9959}  & \multicolumn{1}{c|}{0.996}   & \multicolumn{1}{c|}{0.996}   & \multicolumn{1}{c|}{0.959}   & \multicolumn{1}{c|}{0.907}   & \multicolumn{1}{c|}{0.803}   & \multicolumn{1}{c|}{0.771}   & \multicolumn{1}{c|}{0.666}   & \multicolumn{1}{c|}{0.557}   & \multicolumn{1}{c|}{0.505}   \\ \cline{2-12}
\multicolumn{1}{|c|}{}                           & \multicolumn{1}{c|}{MSE}         & \multicolumn{1}{c|}{5.4E-04} & \multicolumn{1}{c|}{5.8E-04} & \multicolumn{1}{c|}{6.9E-04} & \multicolumn{1}{c|}{1.1E-03} & \multicolumn{1}{c|}{1.7E-03} & \multicolumn{1}{c|}{2.2E-03} & \multicolumn{1}{c|}{3.5E-03} & \multicolumn{1}{c|}{4.2E-03} & \multicolumn{1}{c|}{6.4E-03} & \multicolumn{1}{c|}{6.9E-03} \\ \cline{2-12}
\multicolumn{1}{|c|}{}                           & \multicolumn{1}{c|}{attack iter} & \multicolumn{1}{c|}{61.8}    & \multicolumn{1}{c|}{63.8}    & \multicolumn{1}{c|}{66.5}    & \multicolumn{1}{c|}{72.4}    & \multicolumn{1}{c|}{78.3}    & \multicolumn{1}{c|}{95.3}    & \multicolumn{1}{c|}{113.7}   & \multicolumn{1}{c|}{114.1}   & \multicolumn{1}{c|}{114.3}   & \multicolumn{1}{c|}{114.4}   \\ \hline
\multicolumn{1}{l}{}                             & \multicolumn{1}{l}{}             & \multicolumn{1}{l}{}         & \multicolumn{1}{l}{}         & \multicolumn{1}{l}{}         & \multicolumn{1}{l}{}         & \multicolumn{1}{l}{}         & \multicolumn{1}{l}{}         & \multicolumn{1}{l}{}         & \multicolumn{1}{l}{}         & \multicolumn{1}{l}{}         & \multicolumn{1}{l}{}         \\ \hline
\multicolumn{1}{|c|}{\multirow{4}{*}{\rotatebox{90}{LFW}}}       & \multicolumn{1}{c|}{ASR}         & \multicolumn{1}{c|}{1}       & \multicolumn{1}{c|}{1}       & \multicolumn{1}{c|}{1}       & \multicolumn{1}{c|}{1}       & \multicolumn{1}{c|}{0.97}    & \multicolumn{1}{c|}{0.91}    & \multicolumn{1}{c|}{0.85}    & \multicolumn{1}{c|}{0.78}    & \multicolumn{1}{c|}{0.39}    & \multicolumn{1}{c|}{0.07}    \\ \cline{2-12}
\multicolumn{1}{|c|}{}                           & \multicolumn{1}{c|}{SSIM}        & \multicolumn{1}{c|}{0.998}   & \multicolumn{1}{c|}{0.996}   & \multicolumn{1}{c|}{0.981}   & \multicolumn{1}{c|}{0.976}   & \multicolumn{1}{c|}{0.898}   & \multicolumn{1}{c|}{0.811}   & \multicolumn{1}{c|}{0.659}   & \multicolumn{1}{c|}{0.573}   & \multicolumn{1}{c|}{0.471}   & \multicolumn{1}{c|}{0.41}    \\ \cline{2-12}
\multicolumn{1}{|c|}{}                           & \multicolumn{1}{c|}{MSE}         & \multicolumn{1}{c|}{2.2E-04} & \multicolumn{1}{c|}{4.3E-04} & \multicolumn{1}{c|}{6.8E-04} & \multicolumn{1}{c|}{8.6E-04} & \multicolumn{1}{c|}{1.8E-03} & \multicolumn{1}{c|}{2.8E-03} & \multicolumn{1}{c|}{3.8E-03} & \multicolumn{1}{c|}{4.3E-03} & \multicolumn{1}{c|}{5.5E-03} & \multicolumn{1}{c|}{6.5E-03} \\ \cline{2-12}
\multicolumn{1}{|c|}{}                           & \multicolumn{1}{c|}{attack iter} & \multicolumn{1}{c|}{25}      & \multicolumn{1}{c|}{34.7}    & \multicolumn{1}{c|}{42.5}    & \multicolumn{1}{c|}{68.3}    & \multicolumn{1}{c|}{94.2}    & \multicolumn{1}{c|}{95.5}    & \multicolumn{1}{c|}{95.6}    & \multicolumn{1}{c|}{98.3}    & \multicolumn{1}{c|}{97.9}    & \multicolumn{1}{c|}{98.1}    \\ \hline
\multicolumn{1}{l}{}                             & \multicolumn{1}{l}{}             & \multicolumn{1}{l}{}         & \multicolumn{1}{l}{}         & \multicolumn{1}{l}{}         & \multicolumn{1}{l}{}         & \multicolumn{1}{l}{}         & \multicolumn{1}{l}{}         & \multicolumn{1}{l}{}         & \multicolumn{1}{l}{}         & \multicolumn{1}{l}{}         & \multicolumn{1}{l}{}         \\ \hline
\multicolumn{1}{|c|}{\multirow{4}{*}{\rotatebox{90}{MNIST}}}     & \multicolumn{1}{c|}{ASR}         & \multicolumn{1}{c|}{1}       & \multicolumn{1}{c|}{0.82}    & \multicolumn{1}{c|}{0.57}    & \multicolumn{1}{c|}{0.44}    & \multicolumn{1}{c|}{0.25}    & \multicolumn{1}{c|}{0.06}    & \multicolumn{1}{c|}{0}       & \multicolumn{1}{c|}{0}       & \multicolumn{1}{c|}{0}       & \multicolumn{1}{c|}{0}       \\ \cline{2-12}
\multicolumn{1}{|c|}{}                           & \multicolumn{1}{c|}{SSIM}        & \multicolumn{1}{c|}{0.99}    & \multicolumn{1}{c|}{0.982}   & \multicolumn{1}{c|}{0.974}   & \multicolumn{1}{c|}{0.963}   & \multicolumn{1}{c|}{0.954}   & \multicolumn{1}{c|}{0.935}   & \multicolumn{1}{c|}{0.583}   & \multicolumn{1}{c|}{0.576}   & \multicolumn{1}{c|}{0.581}   & \multicolumn{1}{c|}{0.574}   \\ \cline{2-12}
\multicolumn{1}{|c|}{}                           & \multicolumn{1}{c|}{MSE}         & \multicolumn{1}{c|}{1.5E-05} & \multicolumn{1}{c|}{2.3E-04} & \multicolumn{1}{c|}{2.8E-04} & \multicolumn{1}{c|}{1.2E-03} & \multicolumn{1}{c|}{1.5E-03} & \multicolumn{1}{c|}{2.4E-03} & \multicolumn{1}{c|}{1.7E-02} & \multicolumn{1}{c|}{1.7E-02} & \multicolumn{1}{c|}{1.7E-02} & \multicolumn{1}{c|}{1.7E-02} \\ \cline{2-12}
\multicolumn{1}{|c|}{}                           & \multicolumn{1}{c|}{attack iter} & \multicolumn{1}{c|}{11.5}    & \multicolumn{1}{c|}{34.7}    & \multicolumn{1}{c|}{93.2}    & \multicolumn{1}{c|}{96.7}    & \multicolumn{1}{c|}{97.1}    & \multicolumn{1}{c|}{96.5}    & \multicolumn{1}{c|}{300}     & \multicolumn{1}{c|}{300}     & \multicolumn{1}{c|}{300}     & \multicolumn{1}{c|}{300}     \\ \hline
\end{tabular}
}}
\caption{\small Mitigation with controlled local training iterations}
\label{table:defense_iteration}
\vspace{-0.4cm}
\end{table}

\begin{figure}[ht]
 \centerline{\includegraphics[scale=.37]{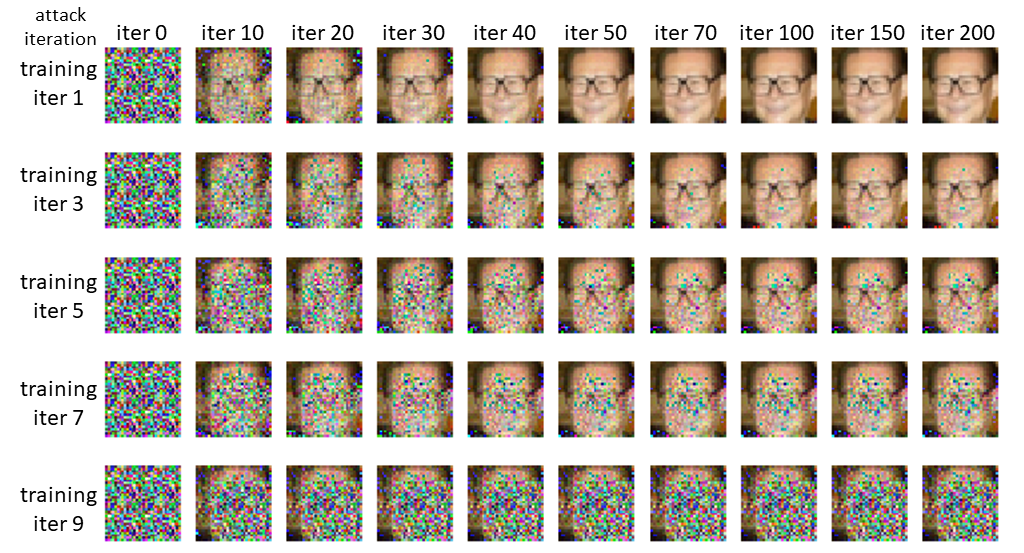}}
  \caption{\small Effect of local training on CPL-patterned attacks using LFW}
  \label{fig:local_iter_large}
\vspace{-0.4cm}
\end{figure}


\noindent \textbf{Gradient squeezing with controlled local training iterations.} Instead of sharing the gradient from the local training computation at each round $t$, we schedule and control the sharing of the gradient only after $M$ iterations of local training. We perform CPL attack after each local iteration. Table~\ref{table:defense_iteration} shows the results of varying $M$ from 1 to 10 with step 1. It shows that as $M$ increases, the ASR of CPL attack starts to decrease, with 97.1\%, 83.5\%, 50\% and 29.2\% for $M =$ 3, 5, 8 and 10 respectively for CIFAR10, and with 100\%, 97\%, 78\% and 7\% for $M =$ 3, 5, 8 and 10 respectively for LFW. This preliminary mitigation study shows that clients in federated learning may adopt some attack resilient optimizations when configuring their local training hyperparameters.  An example of gradient squeezing with controlled local training iterations is provided in Figure~\ref{fig:local_iter_large}.



\vspace{-0.4cm}
\section{Related work}

Privacy in federated learning has been studied in two contexts: training-phase privacy attacks and prediction-phase privacy attacks. Gradient leakage attacks, formulated in CPL of different forms, or those in literature~\cite{geiping2020inverting,zhu2019deep,zhao2020idlg}, are one type of privacy exploits in the training phase. In addition, Aono et al~\cite{aono2017privacy,aono2017privacyc} proposed a privacy attack, which partially recovers private data based on the proportionality between the training data and the gradient updates in multi-layer perceptron models. However, their attack is not suitable for convolutional neural networks because the size of the features is far larger than the size of convolution weights. Hitaj et al~\cite{hitaj2017deep} poisoned the shared model by introducing mislabeled samples into the training. In comparison, gradient leakage attacks are more aggressive since client privacy leakage attacks make no assumption on direct access to the training data as those in training poisoning attacks and yet can compromise the private training data by reconstruction attack based on only the local parameter updates to be shared with the federated server.

Privacy exploits at the prediction phase include model inversion, membership inference, and GAN-based reconstruction attack~\cite{hayes2017logan,hitaj2017deep,wang2019beyond}. Fredrikson et al~\cite{fredrikson2015model} proposed the model inversion attack to exploit confidence values revealed along with predictions. Ganju et al~\cite{ganju2018property} infers the global properties of the training data from a trained white-box fully connected neural network. Membership attacks~\cite{shokri2017membership,melis2019exploiting,stacey2020federated} exploit the statistical differences between the model prediction on its training set and the prediction on unseen data to infer the membership of training data with high confidence.


\vspace{-0.4cm}
\section{Conclusion}

We have presented a principled framework for evaluating and comparing different forms of client privacy leakage attacks. We show how adversaries can reconstruct the private local training data by simply analyzing the shared parameter update (e.g., local gradient or weight update vector). We identify and analyze the impact of different attack configurations and different hyperparameter settings of federated learning on client privacy leakage attacks. Extensive experiments on four benchmark datasets highlight the importance of providing a systematic evaluation framework for an in-depth understanding of the various forms of client privacy leakage threats in federated learning and for developing and evaluating different mitigation strategies.

\bigskip\noindent\textbf{Acknowledgements.} The authors acknowledge the partial support from the National Science
Foundation under Grants SaTC 1564097, NSF 1547102, and an IBM
Faculty Award.

\clearpage

\bibliographystyle{splncs04}

\section{Appendices}

\subsection{Proof of Theorem 1}

\begin{assumption}
(\textbf{Convexity}) we say $f(x)$ is convex if
\begin{equation}
f(\alpha x + (1-\alpha)x') \leq \alpha f(x) + (1-\alpha) f(x'),
\label{equa1}
\end{equation}
where $x,x'$ are data point in $\mathbb{R}^d$, and $\alpha \in [0,1]$.
\label{assumption1}
\end{assumption}

\begin{lemma}
If a convex $f(x)$ is differentiable, we have:
\begin{equation}
f(x') -f(x) \geq \langle \nabla f(x), x'-x \rangle.
\label{equa2}
\end{equation}
\label{lemma1}
\end{lemma}

\begin{proof}
 Equation~\ref{equa1} can be rewritten as:
 $$\frac{f(x'+\alpha(x-x'))-f(x')}{\alpha} \leq f(x)-f(y).$$

 When $\alpha \rightarrow 0$, we complete the proof.
\end{proof}

\begin{assumption}
(\textbf{Lipschitz Smoothness}) With Lipschitz continuous on the differentiable function $f(x)$, we have:
\begin{equation}
|| \nabla f(x) - \nabla f(x') \leq L||x-x'||,
\label{equa3}
\end{equation}
where $L$ is called Lipschitz constant.
\label{assumption2}
\end{assumption}

\begin{lemma}
If $f(x)$ is Lipschitz-smooth, we have:
\begin{equation}
f(x^{t+1}) -f(x^{t}) \leq - \frac{1}{2L}||\nabla f(x^T)||^2_2
\label{equa4}
\end{equation}
\label{lemma2}
\end{lemma}

\begin{proof}
Using the Taylor expansion of $f(x)$ and the uniform bound over Hessian matrix, we have
\begin{equation}
f(x') \leq f(x) + \langle \nabla f(x), x'-x \rangle + \frac{L}{2}||x'-x||^2_2.
\label{equa5}
\end{equation}
By inserting $x' = x - \frac{1}{L}\nabla f(x)$ into equation~\ref{equa3} and equation~\ref{equa5}, we have:
\begin{align*}
    f(x - \frac{1}{L}\nabla f(x)) -f(x) &\leq - \frac{1}{L} \langle \nabla f(x), \nabla f(x) \rangle + \frac{L}{2}||\frac{1}{L} \nabla f(x)||^2_2 \\
    &= - \frac{1}{2L}||\nabla f(x)||^2_2
\end{align*}
\end{proof}

\begin{lemma} \textbf{(Co-coercivity)}
A convex and Lipschitz-smooth $f(x)$ satisfies:
\begin{equation}
\langle \nabla f(x') - \nabla f(x), x'-x \rangle \geq \frac{1}{L} || \nabla f(x') - \nabla f(x)||
\label{equa7}
\end{equation}
\label{lemma3}
\end{lemma}

\begin{proof}
Due to equation~\ref{equa3},
\begin{align*}
 \langle \nabla f(x') - \nabla f(x), x'-x \rangle &\geq    \langle \nabla f(x') - \nabla f(x), \frac{1}{L}(\nabla f(x') - \nabla f(x)) \rangle \\
 &= \frac{1}{L} || \nabla f(x') - \nabla f(x)||
\end{align*}

\end{proof}

Then we can proof the attack convergence theorem:  $f(x^T)-f(x^*) \leq  \frac{2L||x^0-x^*||^2}{T}.$

\begin{proof}

Let $f(x)$ be convex and Lipschitz-smooth. It follow that
\begin{align}
    ||x^{t+1}-x^*||^2_2 &= ||x^t-x^*-\frac{1}{L} \nabla f(x^t)||^2_2 \nonumber \\
    &= ||x^t-x^*||^2_2 - 2\frac{1}{L}\langle x^t-x^*, \nabla f(x^t) \rangle + \frac{1}{L^2}||\nabla f(x^t)||^2_2 \nonumber \\
    &\leq ||x^t-x^*||^2_2 - \frac{1}{L^2}||\nabla f(x^t)||^2_2 \label{equa8}
\end{align}

Equation~\ref{equa8} holds due to equation~\ref{equa7} in lemma~\ref{lemma3}. Recall equation~\ref{equa4} in lemma~\ref{lemma2}, we have:
\begin{equation}
    f(x^{t+1}) -f(x^*) \leq f(x^{t}) -f(x^*) - \frac{1}{2L}||\nabla f(x^t)||^2_2.
     \label{equa9}
\end{equation}

By applying convexity,
\begin{align}
    f(x^t)-f(x^*) &\leq \langle \nabla f(x^t), x^t-x^* \rangle \nonumber \\
&\leq  ||\nabla f(x^t)||_2||x^t-x^*|| \nonumber \\
&\leq ||\nabla f(x^t)||_2||x^1-x^*||. \label{equa10}
\end{align}

Then we insert equation~\ref{equa10} into equation~\ref{equa9}:
\begin{align}
 &f(x^{t+1}) -f(x^*) \leq f(x^{t}) -f(x^*) -  \frac{1}{2L}\frac{1}{||x^1-x^*||^2}( f(x^t)-f(x^*))^2 \nonumber \\
 &\Rightarrow \frac{1}{f(x^{t}) -f(x^*)} \leq \frac{1}{f(x^{t+1}) -f(x^*)} - \beta \frac{f(x^t)-f(x^*)}{f(x^{t+1})-f(x^*)} \label{equa11} \\
 & \Rightarrow \frac{1}{f(x^{t}) -f(x^*)} \leq \frac{1}{f(x^{t+1}) -f(x^*)} - \beta \label{equa12} \\
 & \Rightarrow \beta \leq \frac{1}{f(x^{t+1}) -f(x^*)} - \frac{1}{f(x^{t}) -f(x^*)},
\end{align}
where $\beta = \frac{1}{2L}\frac{1}{||x^1-x^*||^2}$. Equation~\ref{equa11} is done by divide both side with $(f(x^{t+1}) -f(x^*))(f(x^{t}) -f(x^*))$ and Equation~\ref{equa12} utilizes $f(x^{t+1}) -f(x^*) \leq f(x^{t}) -f(x^*)$.
Then, following by induction over $t=0,1,2,..T-1$ and telescopic cancellation, we have $$T\beta \leq \frac{1}{f(x^{T}) -f(x^*)} - \frac{1}{f(x^{0}) -f(x^*)} \leq \frac{1}{f(x^{T}) -f(x^*)}.$$
\begin{align}
 &T\beta \leq \frac{1}{f(x^{T}) -f(x^*)} - \frac{1}{f(x^{0}) -f(x^*)} \leq \frac{1}{f(x^{T}) -f(x^*)}  \\
 &\Rightarrow \frac{T}{2L}\frac{1}{||x^1-x^*||^2} \leq \frac{1}{f(x^{T}) -f(x^*)} \\
  &\Rightarrow f(x^{T}) -f(x^*) \leq \frac{2L||x^0-x^*||^2}{T}.
\end{align}
Thus complete the proof.
\end{proof}

\end{document}